\documentclass[a4j]{article}
\usepackage{authblk}
\usepackage[T1]{fontenc}
\usepackage{lmodern}
\usepackage{textcomp}
\usepackage{latexsym}
\usepackage{comment}
\usepackage{amsmath}
\usepackage{amssymb}
\usepackage{amsfonts}
\usepackage{mathrsfs} %for mathscr
\usepackage{amsthm}
\usepackage{listings}
\usepackage[dvipdfmx]{graphicx}
\usepackage[subrefformat=parens]{subcaption}
\usepackage[top=30truemm,bottom=30truemm,left=25truemm,right=25truemm]{geometry}
\usepackage{here}
\newtheorem{thm}{Theorem}[section]
\newtheorem{defi}{Definition}[section]
\newtheorem{lem}{Lemma }[section]
\newtheorem{cor}{Corollary }[section]

\newtheorem{prop}{Proposition}[section]

\title{Bayesian Generalization Error in Linear Neural Networks with Concept Bottleneck Structure and Multitask Formulation}
\author[1]{Naoki Hayashi\thanks{naoki.hayashi05@aisin.co.jp}}
\author[1]{Yoshihide Sawada\thanks{yoshihide.sawada@aisin.co.jp}}
\affil[1]{Tokyo Research Center, Aisin Corporation}
\date{2023/03/16}



\begin{document}
 \columnseprule=0.3mm
\maketitle
\begin{abstract}
Concept bottleneck model (CBM) is a ubiquitous method that can interpret neural networks using concepts.
In CBM, concepts are inserted between the output layer and the last intermediate layer as observable values.
This helps in understanding the reason behind the outputs generated by the neural networks:
the weights corresponding to the concepts from the last hidden layer to the output layer.
However, it has not yet been possible to understand the behavior of the generalization error in CBM
since a neural network is a singular statistical model in general.
When the model is singular, a one to one map from the parameters to probability distributions cannot be created.
This non-identifiability makes it difficult to analyze the generalization performance.
In this study, we mathematically clarify the Bayesian generalization error and free energy of CBM when its architecture is three-layered linear neural networks.
We also consider a multitask problem where the neural network outputs not only the original output but also the concepts.
The results show that CBM drastically changes the behavior of the parameter region and the Bayesian generalization error in three-layered linear neural networks as compared with the standard version,
whereas the multitask formulation does not.
\end{abstract}

\section{Introduction}
\label{sec:Intro}
Artificial neural networks have been advancing and widely applied in many fields
since the time when multi-layer perceptrons first emerged \cite{GoodfellowDLBook2016, Dong2021DLSurvey}.
However, since most neural networks are black-boxes, interpreting the outputs of neural networks is necessary.
Hence, various procedures have been proposed to improve output interpretability \cite{Molnar2020IMLBook}.
%Artificial neural networks have kept on advancing and been widely applied since they were multi-layer perceptrons \cite{GoodfellowDLBook2016, Dong2021DLSurvey}.
%Applications of them spread in many fields.
%That causes the necessary of interpretation for outputs of neural networks; thus,
%various procedures have been proposed \cite{Molnar2020IMLBook}.
One of the network architectures used to explain the behaviors of neural networks is the concept bottleneck model (CBM) \cite{Kumar2009CBM, Lampert2009CBM, Koh2020CBM20a}.
CBM has a novel structure, called a concept bottleneck structure,
where the concepts are inserted between the output layer,
and the last intermediate layer as observed values and the last connection from the concepts to the output is linear.
%We call this structure concept bottleneck structure.
Thus, humans are expected to be able to interpret the weights of the last connection as the effect of the specified concept to the output,
similar to coefficients of linear regression.
For instance, following \cite{Koh2020CBM20a}, when we predict the knee arthritis grades of patients by using x-ray images and a CBM,
we set the concepts as clinical findings corrected by medical doctors,
and thereby understand how clinical findings affect the predicted grades based on the correlation,
by observing the learned weights in the last connection.
Concept-based interpretation is used in knowledge discovery for chess \cite{McGrath2022AcquisChessAlphaZero}, video representation \cite{Qian2022StatDynamCBMVideoRep}, medical imaging \cite{Hu2022XMIR}, clinical risk prediction \cite{Aniruddh2021MLClinicalRiskPred}, computer aided diagnosis \cite{Klimiene2022MultiviewCBM}, and other healthcare domain problems \cite{Chen2021EthicalMLSurveyHealthcare}.
CBM is a significant foundation for these applications, and advanced methods \cite{Sawada2022CBMAUC, Qian2022StatDynamCBMVideoRep, Klimiene2022MultiviewCBM} have been proposed based on CBM.
Hence, it is important to clarify the theoretical behavior of CBM.
%in particular, comparing it with the case CBM structure is not added.
%We call a neural network without concept bottleneck structure a standard model (Standard) according to \cite{Koh2020CBM20a}.
%For the ease of simplicity, we call a neural network with concept bottleneck structure just CBM.

Multitask formulation (Multitask) \cite{Xu2020MultiTaskConcept} is also needed to clarify the performance difference as compared to CBM
because the former can output a vector concatenating the original output and concepts instead of inserting concepts into the intermediate layer;
CBM and Multitask use similar types of data (inputs, concepts, and outputs).
Their interpretations are as well;
CBM obtains explanation based on regression between concepts and outputs,
and Multitask obtains these from their co-occurrence.

Although some limitations of CBM have been investigated \cite{Mahinpei2021BBCLM, Margeloiu2021CBMIntended, Lockhart2022CBMLeakage},
its generalization error has not yet been clarified except for a simple analysis
that was conducted using the least squares method of a three-layered linear and independent CBM in \cite{Koh2020CBM20a}.
That of Multitask also remains unknown.
This is because, in general, neural networks are non-identifiable, i.e. it is impossible to map from the parameter to the probabilistic distribution which represents the model.
One calls such a model a singular statistical model \cite{SWatanabeBookE, SWatanabeBookMath},
since any normal distribution cannot approximate its likelihood and posterior distribution and its Fisher information matrix has non-positive eigenvalues.
For singular statistical models, it has been proved that Bayesian inference is a better learning method than maximum likelihood or posterior estimation in terms of generalization performance \cite{SWatanabeBookE, SWatanabeBookMath}.
In the following, we therefore mainly consider Bayesian inference.
%%以下はLDA論文と重複を含むので書き直しも必要

A regular statistical model is defined as having parameters that are injective to probability density functions.
This situation is stated as the one in which the model is regular.
Otherwise, the model is singular.
Let $d$ be the parameter dimension and $n$ be the sample size.
In a regular statistical model, its expected generalization error is asymptotically equal to $d/2n + o(1/n)$, where the generalization error is the Kullback-Leibler divergence from the data-generating distribution to the predictive distribution \cite{AkaikeAICconf, Akaike1974AIC, AkaikeAIC}.
Moreover, its negative log marginal likelihood (a.k.a. free energy) has asymptotic expansion represented by $nS_n + (d/2) \log n + O_p(1)$, where $S_n$ is the empirical entropy \cite{Schwarz1978BIC}.
In general case, i.e. the case models can be singular, Watanabe had proved that the asymptotic form of its generalization error $G_n$ and free energy $F_n$ are the followings \cite{Watanabe1,Watanabe2,SWatanabeBookE}:
\begin{align}
\label{thm-watanabeG}
\mathbb{E}_{n}[G_n] &= \frac{\lambda}{n} - \frac{m-1}{n \log n} + o\left(\frac{1}{n \log n} \right), \\
\label{thm-watanabeF}
F_n &= nS_n + \lambda \log n - (m-1) \log\log n + O_p(1),
\end{align}
where $\lambda$ is a positive rational number, $m$ is a positive integer, and $\mathbb{E}_{n}[\cdot]$ is an expectation operator on the overall dataset, respectively.
The constant $\lambda$ is called a learning coefficient since it is dominant in the leading term of (\ref{thm-watanabeG}) and (\ref{thm-watanabeF}),
which represent the $\mathbb{E}_{n}[G_n]$-$n$ and $F_n$-$n$ learning curves.
The above forms hold not only in the case where the model is regular but also in the case where the model is singular.
%Let $S$ be the entropy of the data-generating distribution.
The generalization loss is defined by the cross entropy between the data-generating distribution and the statistical model and equal to $S+G_n$,
where $S$ is the entropy of the data-generating distribution.
Watanabe developed this theory and proposed two model-evaluation methods, WAIC \cite{WatanabeAIC} and WBIC \cite{WatanabeBIC},
which can estimate $S+G_n$ and $F_n$ of regular and singular models from the model and data, respectively.

Let $K:\mathcal{W}\to \mathbb{R}$, $w \mapsto K(w)$ be the Kullback-Leibler (KL) divergence between the data-generating distribution to the statistical model,
where $\mathcal{W} \subset \mathbb{R}^d$ is the parameter set and $w \in \mathcal{W}$.
Assume that $\mathcal{W}$ is a sufficiently large compact set, $K(w)$ is analytic, and its range is non-negative.
The constants $\lambda$ and $m$ are characterized by singularities in the set of the zero points of the KL divergence: $K^{-1}(0)$, which is an analytic set (a.k.a. algebraic variety).
They can be calculated by the resolution of singularities \cite{Hironaka, Atiyah1970resolution};
%$K^{-1}(0)$ is an analytic set (a.k.a. algebraic variety).
$\lambda$ is called a real log canonical threshold (RLCT) and $m$ is called a multiplicity in algebraic geometry.
Note that they are birational invariants; $\lambda$ and $m$ do not depend on how singularities can be resolved.
Here, suppose the prior is positive and bounded on $K^{-1}(0)$.
In the regular case, we can derive $\lambda=d/2$ and $m=1$.
Besides, the predictive distribution can not only be Bayesian posterior predictive but also the model whose parameters are maximum likelihood or posterior estimator.
However, in the singular case, the RLCT $\lambda$ and the multiplicity $m$ depend on the model,
and Bayesian inference is significantly different from the maximum likelihood or posterior estimation.
These situations can occur in both CBM and Multitask.
%CBM and Multitask have same situation.

Determining RLCTs is important for estimating the sufficient sample size, constructing procedures of learning, and selecting models.
In fact, RLCTs of many singular models have been studied:
mixture models \cite{Yamazaki1, Yamazaki2004BinMixRLCT, SatoK2019PMM, Matsuda1-e, WatanabeT2022MultiMixRLCT},
Boltzmann machines \cite{Yamazaki4,Aoyagi2,Aoyagi3},
non-negative matrix factorization \cite{nhayashi2,nhayashi5,nhayashi8},
latent class analysis \cite{Drton2009LCARLCT},
latent Dirichlet allocation \cite{nhayashi7, nhayashi9},
naive Bayesian networks \cite{Rusakov2005asymptotic},
Bayesian networks \cite{Yamazaki3},
Markov models \cite{Zwiernik2011asymptotic},
hidden Markov models \cite{Yamazaki2},
linear dynamical systems \cite{Naito2014KFRLCT},
Gaussian latent tree and forest models \cite{Drton2017forest},
and three-layered neural networks whose activation function is linear \cite{Aoyagi1}, analytic-odd (like $\tanh$) \cite{Watanabe2}, and Swish \cite{Tanaka2020SwishNNRLCT}.
Additionally, a model selection method called sBIC, which uses RLCTs of statistical models, was proposed by Drton and Plummer \cite{Drton}.
%In addition, Drton and Imai have empirically verified that sBIC is more precise than WBIC in the sense of accuracy to select the accurate model when the RLCTs are exactly clarified or their tight bounds are available 
Furthermore, Drton and Imai empirically demonstrated that sBIC is more precise than WBIC in terms of selecting the correct model when the RLCTs are precisely clarified or their tight bounds are available \cite{Drton, Drton2017forest, Imai2019estimating}.
In addition, Imai proposed an estimating method for RLCTs from data and model and extended sBIC \cite{Imai2019estimating}.
Other application of RLCTs is a design procedure for exchange probability in the exchange Monte Carlo method proposed by Nagata \cite{Nagata2008asymptotic}.

The RLCT of a neural network without the concept bottleneck structure, called Standard in \cite{Koh2020CBM20a},
is exactly clarified in \cite{Aoyagi1} in the case of a three-layered linear neural network.
However, the RLCTs of CBM and Multitask remain unknown.
In other words, even if the model structure is three-layered linear, Bayesian generalization errors and marginal likelihoods in CBM and Multitask have not yet been clarified.
Furthermore, since we treat models in the interpretable machine learning field, our result suggests
that singular learning theory is useful for establishing a foundation of responsible artificial/computational intelligence.
This is because interpretation methods can often be referred to restrictions of the parameter space for the model.
Parameter constrain might have to change whether the model is regular/singular \cite{DrtonBook2008LecAlgStat}.
There are perspectives focusing parameter restriction to analyze singular statistical models \cite{DrtonBook2008LecAlgStat, NHayashiPhDThesis}.
For example, in non-negative matrix factorization (NMF) \cite{Paatero, Lee, Cemgil}, parameters are elements of factorized matrices and they are restricted to non-negative regions for improvement interpretablity of factorization result,
like purchase factors of customers and degrees of interests for each product gotten from item-user tables \cite{Kohjima, Kohjima2016NMFreview}.
If they could be negative, owing to canceling positive/negative elements, estimating the popularity of products and potential demand of customers would become difficult.
The Bayesian generalization error and the free energy of NMF differently behave from non-restricted matrix factorization \cite{nhayashi2,nhayashi5,nhayashi8, Aoyagi1}.
Therefore, restriction to parameter space on account of interpretation methods essentially affects learning behavior of the model and they can be analyzed by studying singular learning theory.

In this study, we mathematically derive the exact asymptotic forms of the Bayesian generalization error and the free energy by finding the RLCT of the neural network with CBM
if the structure is three-layered linear.
We also clarify the RLCT of Multitask in that case and compare their theoretical behaviors of CBM with that of Multitask and Standard.
The rest of this paper has four parts.
In section \ref{sec:BayesFrame}, we describe the framework of Bayesian inference and how to validate the model if the data-generating distribution is not known.
%In section 3, we explain a mathematical theory of Bayesian inference when the data-generating distribution is unknown.
%The theory name is singular learning theory.
%Applications and other works in singular learning theory are also listed to show the position of our study in the body of knowledge and emphasize novelty and value of that.
In section \ref{sec:Main}, we state Main Theorems.
In section \ref{sec:Expansion}, we expand Main Theorems for categorical variables.
In section \ref{sec:Discuss}, we discuss about this theoretical result.
In section \ref{sec:Conclusion}, we conclude this paper.
Besides, there are two appendices.
In \ref{sec:Singular}, we explain a mathematical theory of Bayesian inference when the data-generating distribution is unknown.
This theory is the foundation of our study and is called singular learning theory.
%Other works and applications in singular learning theory are also listed to show the position of our study in the body of knowledge and emphasize novelty and value of that.
In \ref{sec:Proof}, we prove Main Theorems, their expanded results, and a proposition for comparison of the RLCT of CBM and Multitask.

\section{Framework of Bayesian Inference}
\label{sec:BayesFrame}
Let $X^n = (X_1, \ldots, X_n)$ be a collection of random variables of $n$ independent and identically distributed from a data generating distribution.
The function value of $X_i$ is in $\mathcal{X}$ which is a subset of a finite-dimensional real Euclidean or discrete space.
In this article, the collection $X^n$ is called the dataset or the sample and its element $X_i$ is called the ($i$-th) data.
Besides, let $q: \mathcal{X} \to \mathbb{R}$, $x \mapsto q(x)$, $p(\cdot|w): \mathcal{X} \to \mathbb{R}$, $x \mapsto p(x|w)$, and $\varphi: \mathcal{W} \to \mathbb{R}$, $w \mapsto \varphi(w)$
be the probability densities of a data-generating distribution, a statistical model, and a prior distribution, respectively.
%These domain $\mathcal{X}$ is a subset of a finite-dimensional real Euclidean or discrete space.
%Let $\varphi: \mathcal{W} \to \mathbb{R}$, $w \mapsto \varphi(w)$ be a probability density of a prior distribution.
Note that the parameter $w$ and its set $\mathcal{W}$ are defined as in the above section.

We define a posterior distribution as the distribution whose density is the following function on $\mathcal{W}$:
\begin{equation}
\label{def-posterior}
\varphi^*(w|X^n) = \frac{1}{Z_n}\varphi(w) \prod_{i=1}^n p(X_i|w),
\end{equation}
where $Z_n$ is a normalizing constant used to satisfy the condition $\int \varphi^*(w|X^n)dw=1$:
\begin{equation}
\label{def-marginallikelihood}
Z_n =  \int dw \varphi(w) \prod_{i=1}^n p(X_i|w).
\end{equation}
This is called a marginal likelihood or a partition function. Its negative log value is called a free energy $F_n = -\log Z_n$.
Note that the marginal likelihood is a probability density function of a dataset. The free energy appears in a leading term of the difference between the data-generating distribution and the model in the sense of dataset generating process.
Furthermore, a predictive distribution is defined by the following density function on $\mathcal{X}$:
\begin{equation}
\label{def-predictive}
p^*(x | X^n ) = \int dw \varphi^*(w|X^n) p(x|w).
\end{equation}
This is a probability distribution of a new data.
It is also important for statistics and machine learning to evaluate the dissimilarity between the true and the model in the sense of a new data generating process.

Here, we explain the evaluation criteria for Bayesian inference.
The KL divergence between the data-generating distribution to the statistical model is denoted by
\begin{equation}
\label{def-aveerror}
K(w) = \int dx q(x) \log \frac{q(x)}{p(x|w)}.
\end{equation}
As technical assumptions,
we suppose the parameter set $\mathcal{W} \subset \mathbb{R}^d$ is sufficiently wide and compact and the prior is positive and bounded on
\begin{equation}
\label{algvariety}
K^{-1}(0) := \{ w \in \mathcal{W} \mid K(w)=0 \},
\end{equation}
i.e. $0<\varphi(w)<\infty$ for any $w \in K^{-1}(0)$. 
In addition, we assume that $\varphi(w)$ is a $C^{\infty}$-function on $\mathcal{W}$ and $K(w)$ is an analytic function on $\mathcal{W}$.
An entropy of $q(x)$ and an empirical one are denoted by
\begin{align}
\label{def-entropy}
S &= -\int dx q(x) \log q(x), \\
S_n &= -\frac{1}{n} \sum_{i=1}^n \log q(X_i).
\end{align}
By definition, $X_i \sim q(x)$ and $X^n \sim \prod_{i=1}^n q(x_i)$ hold; thus let $\mathbb{E}_{n}[\cdot]$ be an expectation operator on overall dataset defined by
\begin{equation}
\label{def-expectation}
\mathbb{E}_{n}[\cdot] = \int dx^n \prod_{i=1}^n q(x_i) [\cdot],
\end{equation}
where $dx^n = dx_1\ldots dx_n$.
%$\mathbb{E}_{n}[S_n]$ is $S$.
Then, we have the following KL divergence
\begin{align}
\int dx^n \prod_{i=1}^n q(x_i) \log \frac{\prod_{i=1}^n q(x_i)}{Z_n}
&= -\mathbb{E}_{n}\left[nS_n \right] - \mathbb{E}_{n}[\log Z_n] \\
&= -nS + \mathbb{E}_{n}[F_n],
\end{align}
where $\mathbb{E}_{n}[S_n]$ is $S$.
The expected free energy is an only term that depends on the model and the prior.
For this reason, the free energy is used as a criterion to select the model.
On the other hand, a Bayesian generalization error $G_n$ is defined by the KL divergence between the data-generating distribution and the predictive one:
\begin{equation}
\label{def-gerror}
G_n = \int dx q(x) \log \frac{q(x)}{p^*(x|X^n)}.
\end{equation}

\begin{comment}%%点推定の議論しなさそう……
If we consider maximum likelihood or posterior estimation, we can define the generalization error of these methods by replacing the predictive distribution with the model whose parameter is the estimator $\hat{w}$:
\begin{equation}
\label{def-gerror-mle}
\int dx q(x) \log \frac{q(x)}{p(x|\hat{w})}.
\end{equation}
Formally, if we allow the prior to become a delta function (hyperfunction)
\begin{equation}
\label{def-deltafunc}
\delta(x) = \begin{cases}
\infty & x=0 \\
0 & x \ne 0
\end{cases}
\end{equation}
and $\int \delta(x) dx=1$, the predictive distribution can be $p(x|\hat{w})$ since we can refer the posterior distribution to $\delta(w-\hat{w})$,
where $\hat{w}$ is an estimator of the parameter.
As mentioned below, for constructing the theory studied in this paper, the prior must not depend on the dataset; hence, the above view is just for form's sake but defining $G_n$ as the KL divergence can be extended in the case we plugin $\hat{w}$ instead of the posterior averaging to set the predictive distribution.
\end{comment}

Here, the Bayesian inference is defined by inferring that the data-generating distribution may be the predictive one.
For an arbitrary finite $n$, by the definition of marginal likelihood (\ref{def-marginallikelihood})
and predictive distribution (\ref{def-predictive}), we have
\begin{align}
p^*(X_{n+1}|X^n) &= \frac{1}{Z_n}\int dw \varphi(w) \prod_{i=1}^n p(X_i|w) p(X_{n+1}|w) \\
&= \frac{1}{Z_n} \int dw \varphi(w) \prod_{i=1}^{n+1} p(X_i|w) \\
&= \frac{Z_{n+1}}{Z_n}.
\end{align}
Considering expected negative log values of both sides, according to \cite{SWatanabeBookE}, we get
\begin{align}
\mathbb{E}_{n+1}[-\log p^*(X_{n+1}|X^n)] &= \mathbb{E}_{n+1}[-\log Z_{n+1} - (-\log Z_n) ]\\
\mathbb{E}_{n}[G_n] + S &= \mathbb{E}_{n+1}[F_{n+1}] - \mathbb{E}_{n}[F_n].
\end{align}
Hence, $G_n$ and $F_n$ are important random variables in Bayesian inference
when the data-generating process is unknown.
The situation wherein $q(x)$ is unknown is considered generic \cite{StatRethinkMcElreath2nd, Watanabe2023AdjustCV}.
Moreover, in general, the model can be singular when it has a hierarchical structure or latent variables \cite{SWatanabeBookE, SWatanabeBookMath}.
We therefore investigate how they asymptotically behave in the case of CBM and Multitask. %in the general case of \cite{SWatanabeBookMath}.
For theoretically treating the case when the model is singular, resolution of singularity in algebraic geometry is needed.
Brief introduction of this theory is in \ref{sec:Singular}.
%To establish this theory, resolution of singularity in algebraic geometry has been needed.  

\section{Main Theorems}
%%こっから完全に新しく書き直すことになる
%%12/19: 主結果をちゃんと改めてノート上で整理した方が良い
%%c=g(x)とy=f(g(x))の分離に分布のガウス性を仮定していた、特にc=g(x)の方で。
%%出力については結局分類問題（多クラスも）も同値になるので問題なさそう
%%論文では回帰でも最初から多出力を考えた方が良いかも？（結局正則になる）
\label{sec:Main}
In this section, we state the Main Theorem: the exact value of the RLCTs of CBM and Multitask.
Let $n$ be the sample size, $N$ be the input dimension, $K$ be the number of concepts, and $M$ be the output dimension, respectively.
%In the following, the parameter $w=(A,B)$ is a pair of an $M \times H$ matrix $A$ and an $H \times N$ matrix $B$.
For simplicity, we consider CBM for the regression case and the concept is a real number:
the input data $x$ is an $N$-dimensional vector,
the concept $c$ is an $K$-dimensional vector,
and the output data $y$ is an $M$-dimensional vector, respectively.
We consider the case in which $c$ or $y$ includes categorical variables as described in section \ref{sec:Expansion}.
%However, since their RLCTs are equal, our result also holds for classification case and the case that the concept is binary label.
Let $A=(a_{ik})_{i=1,k=1}^{M,K}$ and $B=(b_{kj})_{k=1,j=1}^{K,N}$ be $M \times K$ and $K \times N$ matrix, respectively.
They are the connection weights of CBM: $y=ABx$ and $c=Bx$.
Similar to CBM, we consider Multitask.
Let $H$ be the number of units in the intermediate layer in this model.
Matrices $U=(u_{ik})_{i=1,k=1}^{M+K,H}$ and $V=(v_{kj})_{k=1,j=1}^{H,N}$ are denoted by connection weights of Multitask: $[y; c]=UVx$, where $[y; c]$ is an $(M+K)$-dimensional vector constructed by concatenating $y$ and $c$ as a column vector,
i.e. putting $z=[y;c]$, we have
\begin{align}
z=(z_h)_{h=1}^{M+K}, \
z_h=
\begin{cases}
y_h & 1 \leqq h \leqq M, \\
c_{h-M+1} & M+1 \leqq h \leqq M+K.
\end{cases}
\end{align}
For the notation, see also Table \ref{params}.
Below, we apply the $[\cdot; \cdot]$ operator to matrices and vectors to vertically concatenate them in the same way as above.
%as vertical concatenating of them in the same way of the above.
%In particular, to reconcile model sizes, we mainly treat the case $H=K$.

%By definition, a set of stochastic matrices is compact. Let $\mathrm{Onehot}(N):=\{w = (w_j) \in \{0,1\}^N \mid \sum_{j=1}^N w_j =1\}$ be a set of $N$-dimensional one-hot vectors and
%$\mathrm{\Delta}_N:=\{c=(c_j)_{j=1}^N \mid \sum_{j=1}^N c_j=1\}$ be an $N$-dimensional simplex.
%In LDA terminology, the number of documents and the vocabulary size is denoted by $N$ and $M$, respectively.
%Let $H_0$ be the optimal (or true) number of topics and $H$ be the chosen one. In this situation, the sample size $n$ is the number of words in all of the given documents. See also Table \ref{params} (this table is quoted and modified from our previous study \cite{nhayashi7}).

\begin{table}[htb]
  \centering
  \caption{Description of the Variables}
  \begin{tabular}{|c|c|c|} \hline
    Variable & Description & Index \\
    \hline\hline
    $b_j=(b_{kj}) \in \mathbb{R}^K$ & connection weights from $x$ to $c$ & for $k=1,\ldots,K$ \\
    $a_k=(a_{ik}) \in \mathbb{R}^M$ & connection weights from $c$ to $y$ & for $i=1,\ldots,M$ \\
    \hline
    $v_j=(v_{kj}) \in \mathbb{R}^H$ & connection weights from $x$ to the middle layer & for $k=1,\ldots,H$ \\
    $u_k=(u_{ik}) \in \mathbb{R}^{M+K}$ & connection weights from the middle layer to $z$ & for $i=1,\ldots,M+K$ \\
    \hline
    $x=(x_j) \in \mathbb{R}^N$ & $j$-th input is $x_j$ & for $j=1,\ldots,N$  \\
    $c=(c_k) \in \mathbb{R}^K$ & $k$-th concept is $c_k$ & for $k=1,\ldots,K$  \\
    $y=(y_i) \in \mathbb{R}^M$ & $i$-th output is $y_i$ & for $i=1, \ldots,M$  \\
    $z=(z_h) \in \mathbb{R}^{M+K}$ & $h$-th output of Multitask is $z_h$ & for $h=1,\ldots,M+K$ \\
    \hline
    $*_0$ and $*^0$ & optimal or true variable corresponding to $*$ & - \\
  \hline
  \end{tabular}
\label{params}
\end{table}

%Set $(A_0, B_0)$ is one of optimal parameters, where $A_0=(a^0_{ik})_{i=1,k=1}^{M,H_0}$ and $B_0=(b^0_{kj})_{k=1,j=1}^{H_0,N}$.

Define the RLCT of CBM and that of Multitask in the below.
We consider neural networks in the case they are three-layered linear. %(a.k.a. reduced rank regression).
First, we state these model structures.
%As mentioned in the following, we principally consider Joint CBM \cite{Koh2020CBM20a}.

\begin{defi}[CBM]
\label{def-CBMNN}
Let $q_1(y,c|x)$ and $p_1(y,c|A,B,x)$ be conditional probability density functions of $(y,c) \in \mathbb{R}^M \times \mathbb{R}^K$ given $x \in \mathbb{R}^N$ as the followings:
\begin{align}
\label{cbm-true}
%q_1(y,c|x) & \propto \exp \left( -\frac{1}{2}\lVert y - A_0 B_0 x \rVert^2 \right)
%\exp \left( -\frac{\gamma}{2}\lVert c - B_0 x \rVert^2 \right), \\
q_1(y,c|x) &= p_1(y,c|A_0, B_0, x), \\
\label{cbm-model}
p_1(y,c|A,B,x) & \propto \exp \left( -\frac{1}{2}\lVert y - A B x \rVert^2 \right)
\exp \left( -\frac{\gamma}{2}\lVert c - B x \rVert^2 \right),
\end{align}
where $A_0=(a^0_{ik})_{i=1,k=1}^{M,K}$ and $B_0=(b^0_{kj})_{k=1,j=1}^{K,N}$ are the true parameters
and $\gamma>0$ is a positive constant controlling task and explanation tradeoff \cite{Koh2020CBM20a}.
The prior density function is denoted by $\varphi_1(A,B)$.
The data-generating distribution of CBM and the statistical model of that are defined by $q_1(y,c|x)$ and $p_1(y,c|A,B,x)$, respectively.
\end{defi}

These distributions are based on the loss function of Joint CBM,
which provides the highest classification performance \cite{Koh2020CBM20a}.
Other types of CBM are discussed in section \ref{sec:Discuss}.
This loss is defined by the linear combination of the loss between $y$ and $ABx$ and that between $c$ and $Bx$.
For regression, these losses are squared Euclidian distances and the linear combination of them are equivalent to negative log likelihoods
$\left(
-\log \prod_{l=1}^n p_1(y^l,c^l|A,B,x^l)
\right)$,
where the data is $(x^l, c^l, y^l)_{l=1}^n \in (\mathbb{R}^N \times \mathbb{R}^K \times \mathbb{R}^M)^n$.
Since CBM assumes that concepts are observable, $c$ is subject to the data-generating distribution and that causes that the number of columns in $A_0$ and rows in $B_0$ are equal to $K$: the number of concepts.

Set a density as $p_{12}(y|A,B,x) \propto \exp \left( -\frac{1}{2}\lVert y - A B x \rVert^2 \right)$.
Then, the statistical model of Standard is $p_{12}(y|A,B,x)$
and the data-generating distribution of Standard can be represented as $q_{12}(y|x):=p_{12}(y|A_0,B_0,x)$;
however, when considering only Standard, the rank of $A_0B_0$ might be smaller than $K$,
i.e. there exists a pair of matrices $(A'_0, B'_0)$ such that $A_0B_0 = A'_0B'_0$ and $\mathrm{rank}A'_0B'_0 < K$.
In other words, if we cannot observe concepts, then a model selection problem can be said to exist: in other words, how do we design the number of middle units in order to find the data-generating distribution or a predictive distribution which realizes high generalization performance.
This problem also appears in Multitask when the number of the middle layer units exceeds the true rank (i.e. the true number of those units): $H>H_0$.
Distributions of Multitask are defined below.

\begin{defi}[Multitask]
\label{def-MultitaskNN}
Put $z=[y; c]$.
Let $q_2(z|x)$ and $p_2(z|U,V,x)$ be conditional probability density functions of $z \in \mathbb{R}^{M+K}$ given $x \in \mathbb{R}^N$ as below:
\begin{align}
\label{multitask-true}
%q_2(z|x) & \propto \exp \left( -\frac{1}{2}\lVert z - U_0 V_0 x \rVert^2 \right), \\
q_2(z|x) &= p_2(z|U_0,V_0,x), \\
\label{multitask-model}
p_2(z|U,V,x) & \propto \exp \left( -\frac{1}{2}\lVert z - U V x \rVert^2 \right),
\end{align}
where $U_0=(v^0_{ik})_{i=1,k=1}^{M+K,H_0}$ and $V_0=(v^0_{kj})_{k=1,j=1}^{H_0,N}$ are the true parameters and $H_0$ is the rank of $U_0V_0$.
The prior density function is denoted by $\varphi_2(U,V)$.
The data-generating distribution of Multitask and the statistical model of that are defined by $q_2(z|x)$ and $p_2(z|U,V,x)$, respectively.
\end{defi}

The data-generating distribution and the statistical model of Standard are defined by $q_2(y|x)$ and $p_2(y|U,V,x)$ of Multitask when $K=0$.

We visualize CBM and Multitask as graphical models in Figures \ref{fig:CBM} and \ref{fig:Multitask}.
In CBM (especially, Joint CBM), the concepts are inserted as observations between the last intermediate layer and the output layer but the connection weights $A$ from $c$ to $y$ are learned based on the relationship $y=ABx$. Concepts insertion is represented as the other part of the model: $c=Bx$.
However, in Multitask, the concepts are concatenated to the output and the connection weights $(U,V)$ are trained as $[y; c] = UVx$. 

\begin{figure}[t]
\begin{minipage}{0.5\hsize}
\begin{flushleft}
\includegraphics[width=8cm,height=4.36cm]{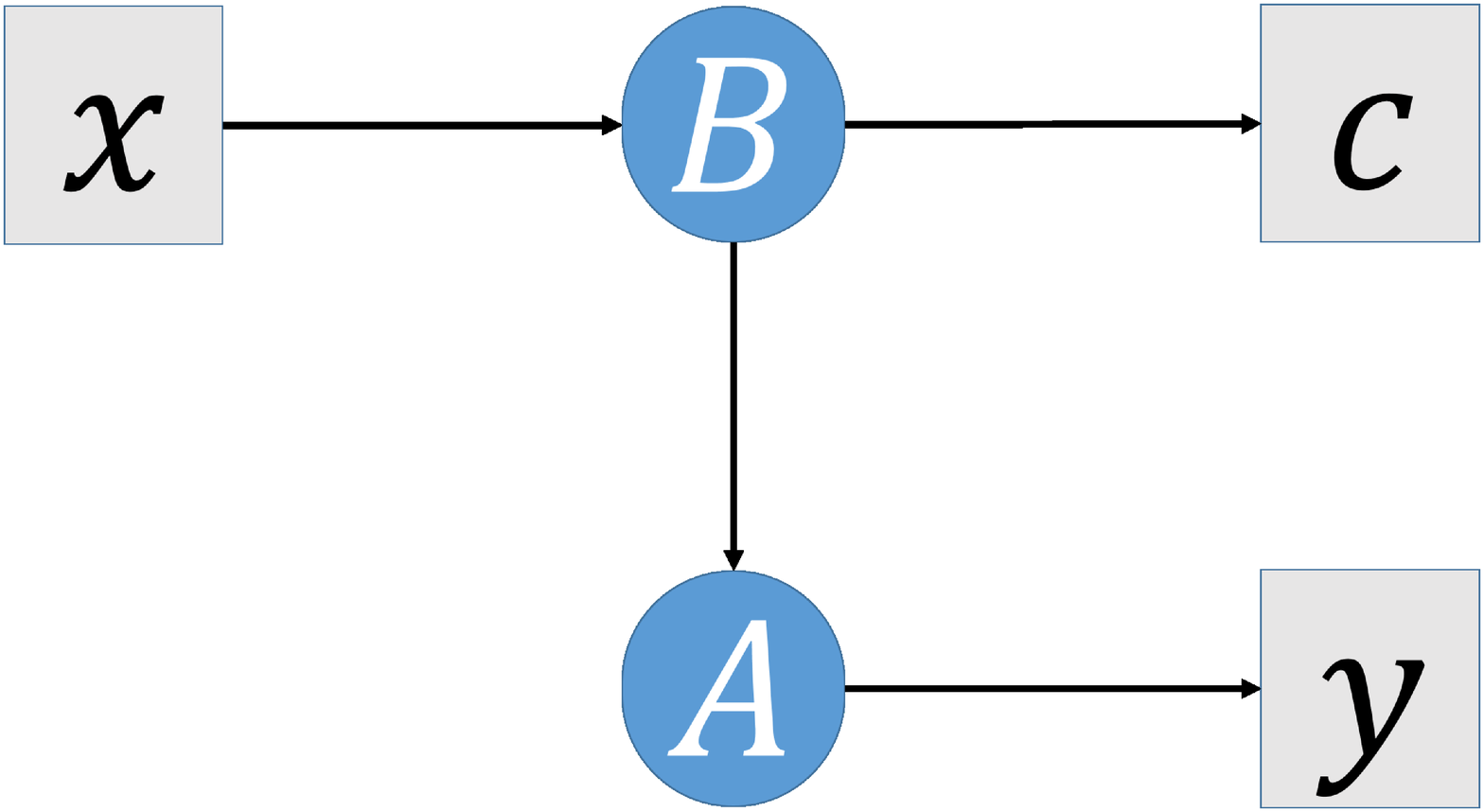}
\subcaption{Graphical Model of CBM}
\label{fig:CBM}
\end{flushleft}
\end{minipage}
\begin{minipage}{0.5\hsize}
\begin{flushright}
\includegraphics[width=8cm,height=4.35cm]{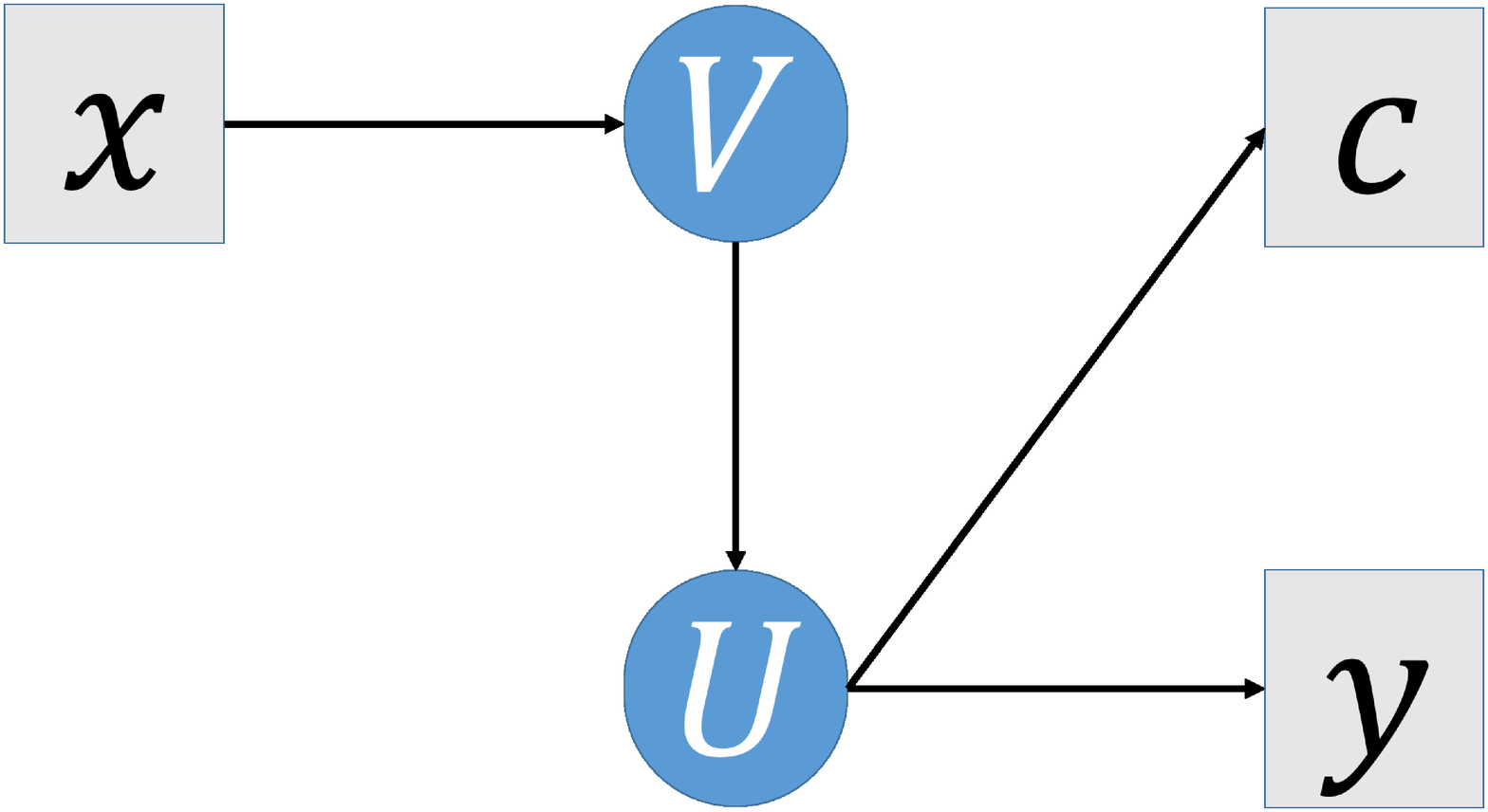}
\subcaption{Graphical Model of Multitask}
\label{fig:Multitask}
\end{flushright}
\end{minipage}
\caption{(a) This figure shows the graphical model of CBM (in particular, Joint CBM) when the neural network is three-layered linear.
Squares $(x,c,y)$ are observed data and circles $(A, B)$ are learnable parameters (connection weights of linear neural network).
Arrows mean that we set the conditional probability model from the left variable to the right one: $p_1(y,c|A,B,x)=p_{12}(y|A,B,x)p_{11}(c|x,B)$, where $p_{12}(y|A,B,x)$ is the statistical model of Standard and $p_{11}(c|x,B)$ is a density function which satisfies $p_{11}(c|B,x) \propto \exp \left(-\frac{\gamma}{2}\lVert c-Bx \rVert^2 \right)$.
\\ (b) This figure states the graphical model of Multitask  in the case the neural network is three-layered linear.
As same as the above, squares $(x,c,y)$ are observations, circles $(U, V)$ are learnable weights, and arrows correspond to conditional probability models, respectively.
For the sake of ease to compare with CBM, we draw $c$ and $y$ as other squares;
however, in fact, they are treated as an output since they are concatenated as a vector: $p_2(z|U,V,x)$, where $z=[y; c]$.
}
\end{figure}

We define the RLCTs of CBM and Multitask as follows.
\begin{defi}[RLCT of CBM]
\label{def-CBMRLCT}
Let $K_1(A,B)$ be the KL divergence between $q_1$ and $p_1$:
\begin{align}
K_1(A,B)=\displaystyle\iiint dydcdx q'(x)q_1(y,c|x) \log \frac{q_1(y,c|x)}{p_1(y,c|A,B,x)},
\end{align}
where $q'(x)$ is the data-generating distribution of the input.
$q'(x)$ is not observed and assumed that it is positive and bounded.
Assume that $\varphi_1(A,B) >0$ is positive and bounded on $K_1^{-1}(0) \ni (A_0,B_0)$.
Then, the zeta function of the learning theory in CBM is
the holomorphic function of univariate complex variable $z \ (\mathrm{Re} (z) >0)$
\begin{align}
\zeta_1(z)=\iint
 K_1(A,B)^z dAdB
\end{align}
and it can be analytically continued to a unique meromorphic function on the entire complex plane $\mathbb{C}$ and all of its poles are negative rational numbers. The RLCT of CBM is defined by $\lambda_1$, where the largest pole of $\zeta_1(z)$ is $(-\lambda_1)$.
Its multiplicity $m_1$ is defined as the order of the maximum pole.
\end{defi}

\begin{defi}[RLCT of Multitask]
\label{def-MultitaskRLCT}
Let $K_2(U,V)$ be the KL divergence between $q_2$ and $p_2$:
\begin{align}
K_2(U,V)=\displaystyle\iiint dydcdx q'(x)q_2(y,c|x) \log \frac{q_2(y,c|x)}{p_2(y,c|U,V,x)},
\end{align}
where $q'(x)$ is same as that of Definition \ref{def-CBMRLCT}.
Assume that $\varphi_2(U,V) >0$ is positive and bounded on $K_2^{-1}(0) \ni (U_0, V_0)$.
As in the case of CBM, the zeta function of learning theory in Multitask is the following holomorphic function of $z \ (\mathrm{Re} (z) >0)$
\begin{align}
\zeta_2(z)=\iint
 K_2(U,V)^z dUdV
\end{align}
and the RLCT of Multitask is defined by $\lambda_2$, where the largest pole of $\zeta_2(z)$ is $(-\lambda_2)$ and its multiplicity $m_2$ is defined as the order of the maximum pole.
\end{defi}

Put an $N \times N$ matrix
$\mathscr{X} = \left( \int x_i x_j q'(x) dx \right)_{i=1,j=1}^{N,N}.$
Then, our main results are the following theorems.

\begin{thm}[Main Theorem 1]
\label{thm-main-cbm}
%Suppose $M\geqq 2$, $N \geqq 2$, and $H \geqq H_0 \geqq 1$.
%Suppose that the $N \times N$ matrix $((\int x_i x_j q'(x) dx)_{ij})_{i=1,j=1}^{N,N}$ is positive definite.
Suppose $\mathscr{X}$ is positive definite and $(A, B)$ is in a compact set.
CBM is a regular statistical model in the case the network architecture is three-layered linear.
Therefore, by using the input dimension $N$, the number of concepts $K$, and the output dimension $M$,
the RLCT of CBM $\lambda_1$ and its multiplicity $m_1$ are as follows:
\begin{align}
\lambda_1 = \frac{1}{2}(M+N)K, \ m_1=1.
\end{align}
\end{thm}

\begin{thm}[Main Theorem 2]
\label{thm-main-multitask}
%Suppose that the $N \times N$ matrix $((\int x_i x_j q'(x) dx)_{ij})_{i=1,j=1}^{N,N}$ is positive definite.
Suppose $\mathscr{X}$ is positive definite and $(U, V)$ is in a compact set.
Let $\lambda_2$ be the RLCT of Multitask formulation of a three-layered linear neural network and $m_2$ be its multiplicity.
By using the input dimension $N$, the number of concepts $K$ and intermediate units $H$, the output dimension $M$, and the true rank $H_0$,
$\lambda_2$ and $m_2$ can be represented as below:
\begin{enumerate}
\item In the case of $M+K+H_0 \leqq N+H$ and $N+H_0 \leqq M+K+H$ and $H+H_0 \leqq N+M+K$,
    \begin{enumerate}
        \item and if $N+M+K+H+H_0$ is even, then
        $$\lambda_2 = \frac{1}{8}\{ 2(H+H_0)(N+M+K)-(N-M-K)^2-(H+H_0)^2 \}, \ m_2 = 1.$$
        \item and if $N+M+K+H+H_0$ is odd, then
        $$\lambda_2 = \frac{1}{8}\{ 2(H+H_0)(N+M+K)-(N-M-K)^2-(H+H_0)^2 +1 \}, \ m_2 = 2.$$
    \end{enumerate}
\item In the case of $N+H<M+K+H_0$, then
$$\lambda_2 = \frac{1}{2}\{HN+H_0(M+K-H)\}, \ m_2 = 1.$$
\item In the case of $M+K+H<N+H_0$, then
$$\lambda_2 = \frac{1}{2}\{H(M+K)+H_0(N-H)\}, \ m_2 = 1.$$
\item Otherwise (i.e. $N+M+K<H+H_0$), then
$$\lambda_2 = \frac{1}{2}N(M+K), \ m_2 = 1.$$
\end{enumerate}
\end{thm}

These theorems yield the exact asymptotic forms of the expected Bayesian generalization error and the free energy following Eqs. (\ref{thm-watanabeG}) and (\ref{thm-watanabeF}).
Their proofs are in \ref{sec:Proof}.

Theorem \ref{thm-main-cbm} shows that the concept bottleneck structure makes the neural network regular if the network architecture is three-layered linear.
We present a sketch of proof below.
Theorem \ref{thm-main-multitask} can be immediately proved using the RLCT of three-layered neural network \cite{Aoyagi1},
since in the model $p_2(z|A,B,x)$, the input dimension, number of middle layer units, output dimension, and rank of the product of true parameter matrices are $N$, $H$, $M+K$, and $H_0$, respectively.

\begin{proof}[Sketch of Proof of Theorem \ref{thm-main-cbm}]
Let a binary relation $\sim$ be that the RLCTs and multiplicities of both sides are equal.
The KL divergence from $q_1$ to $p_1$ can be developed as
\begin{align}
K_1(A,B)
&\propto \displaystyle\iiint dydcdx q'(x)q(y,c|x) \left(
-\lVert y-A_0B_0x \rVert^2 - \gamma\lVert c-B_0x \rVert^2 \right. \\
&\quad \left. + \lVert y-ABx \rVert^2 + \gamma\lVert c-Bx \rVert^2
\right) \\
&\sim \lVert AB-A_0B_0 \rVert^2 + \lVert B-B_0 \rVert^2. 
\end{align}
To calculate $\lambda_1$ and $m_1$, we find $K_1^{-1}(0)$.
Put $\lVert AB-A_0B_0 \rVert^2 + \lVert B-B_0 \rVert^2=0$ and we have $(A,B)=(A_0,B_0)$.
That means that $K_1^{-1}(0)$ can be referred to a one point set; thus, CBM in the three-layered linear case is regular.
\end{proof}

\section{Expansion of Main Theorems}
\label{sec:Expansion}
We define the RLCT of CBM for the observed noise as subject to a Gaussian distribution (cf. Definition \ref{def-CBMNN} and \ref{def-CBMRLCT}).
This corresponds to a regression from X-ray images to arthritis grades in the original CBM study \cite{Koh2020CBM20a}.
However,we can generally treat CBM as classifier and concepts as categorical variables.
For example, in \cite{Koh2020CBM20a}, Koh et al. demonstrated bird species classification task with bird attributes concepts.
%We call $p_{12}(y|A,B,x)$ and $p_{11}(c|B,x)$ task distribution and concept distribution, respectively.
Summarizing them, we have the following four cases:
\begin{enumerate}
\item Both $p_{12}(y|A,B,x)$ and $p_{11}(c|B,x)$ are Gaussian (regression task with real number concepts). \\
\item $p_{12}(y|A,B,x)$ is Gaussian and $p_{11}(c|B,x)$ is Bernoulli (regression task with categorical concepts). \\
\item $p_{12}(y|A,B,x)$ is categorical and $p_{11}(c|B,x)$ is Gaussian (classification task with real number concepts). \\
\item $p_{12}(y|A,B,x)$ is categorical and $p_{11}(c|B,x)$ is Bernoulli (classification task with categorical concepts).
\end{enumerate}
Note that concepts are not exclusive; thus, the distribution of $p_{11}(c|B,x)$ must be Bernoulli (not categorical).
We prove that a result similar to that of Theorem \ref{thm-main-cbm} holds in the above cases.
Before expanding our Main Theorems, we first define the sigmoid and softmax functions.
Let $\sigma_{K'}: \mathbb{R}^{K'} \to [0,1]^{K'}$ be a $K'$-dimensional multivariate sigmoid function and $s_{M'}: \mathbb{R}^{M'} \to \Delta_{M'}$ be a $M'$-dimensional softmax function, respectively:
\begin{align}
\sigma_{K'}(u) &= \left( \frac{1}{1+\exp(-u_j)} \right)_{j=1}^{K'}, \ u \in \mathbb{R}^{K'}, \\
s_{M'}(w) &= \left( \frac{\exp(w_j)}{\sum_{j=1}^{M'} \exp(w_j)} \right)_{j=1}^{M'} \ w \in \mathbb{R}^{M'},
\end{align}
where $\Delta_{M'}$ is an $M'$-dimensional simplex.
Then, we can define each distribution as follows:
\begin{align}
p_{12}^1(y|A,B,x) &\propto \exp\left( -\frac{1}{2}\lVert y - A B x \rVert^2 \right), \\
p_{12}^2(y|A,B,x) &= \prod_{j=1}^{M} (s_{M}(ABx))_j^{y_j}, \\
p_{11}^1(c|B,x) &\propto \exp\left( -\frac{\gamma}{2}\lVert c - B x \rVert^2 \right), \\
p_{11}^2(c|B,x) &\propto \left( \prod_{k=1}^K (\sigma_{K}(Bx))_k^{c_k} (1-(\sigma_{K}(Bx))_k)^{1-c_k} \right)^{\gamma},
\end{align}
where $(s_M(ABx))_j$ and $(\sigma_K(Bx))_k$ are the $j$-th and $k$-th element of $s_M(ABx)$ and $\sigma_K(Bx)$, respectively.
The data-generating distributions are denoted by
$q_{12}^1(y|x) = p_{12}^1(y|A_0,B_0,x)$, $q_{12}^2(y|x) = p_{12}^2(y|A_0,B_0,x)$,
$q_{11}^1(c|x) = p_{11}^1(c|B_0,x)$, and $q_{11}^2(c|x) = p_{11}^2(c|B_0,x)$, respectively.
Based on the aforementioned points, we explain the semantics of the indexes in the density functions $\cdot_{kl}^h$ as follows.
Superscripts $h \in \{1, 2\}$ denote the types of the response variable: real and categorical.
For a double subscript $kl$, $k \in \{1, 2\}$ denotes the models (CBM and Multitask) and $l \in \{1, 2\}$ does response variables ($c$ and $y$).
For a double superscript $ij$, as used in Theorems \ref{thm-expand-cbm} and \ref{thm-expand-multitask},
$i$ and $j$ mean the response variable $y$ and $c$, respectively.
Then, Theorem \ref{thm-main-cbm} can be expanded as follows:
%Then, the expansion of Theorem 3.1 is as the following.

\begin{thm}[Expansion of Theorem \ref{thm-main-cbm}]
\label{thm-expand-cbm}
Let
\begin{align}
p_1^{ij}(y,c|A,B,x) &= p_{12}^i(y|A,B,x)p_{11}^j(c|B,x), \ i=1,2, \ j=1,2, \\
q_1^{ij}(y,c|x) &= p_1^{ij}(y,c|A_0,B_0,x), \ i=1,2, \ j=1,2.
\end{align}
If we write the KL divergences as
\begin{align}
K_1^{ij}(A,B)=\displaystyle\iiint dydcdx q'(x)q_1^{ij}(y,c|x) \log \frac{q_1^{ij}(y,c|x)}{p_1^{ij}(y,c|A,B,x)}, \ i=1,2, \ j=1,2,
\end{align}
where $q'(x)$ is the input-generating distribution (as same as Definition \ref{def-CBMRLCT}).
Assume that $\mathscr{X}$ is positive definite and $(A, B)$ is in a compact set.
Then, the maximum pole $(-\lambda_1^{ij})$ and its order $m_1^{ij}$ of the zeta function
\begin{align}
\zeta_1^{ij}(x) = \iint K_1^{ij}(A,B)^z dAdB
\end{align}
are as follows: for $i=1,2$ and $j=1,2$,
%equal when $(i,j)=(1,1), (1,2), (2,1), (2,2)$.
\begin{align}
\lambda_1^{1j} &= \frac{1}{2}(M+N)K, \\
\lambda_1^{2j} &= \frac{1}{2}(M+N-1)K, \\
m_1^{ij} &= 1.
\end{align}
\end{thm}
Moreover, we expand our Main Theorem \ref{thm-main-multitask} for Multitask.
In general, Multitask also has two task and concept types, same as above.
The dimension of $UVx$ is $M+K$; hence, we can decompose the former $M$-dimensional part and the other ($K$-dimensional) as the same way of $z=[y; c]$.
We define this decomposition as $UVx=[(UVx)^{\mathrm{y}}; (UVx)^{\mathrm{c}}]$,
where $(UVx)^{\mathrm{y}} = ((UVx)_h)_{h=1}^{M}$ and $(UVx)^{\mathrm{c}} = ((UVx)_h)_{h=M+1}^{M+K}$.
%and $(UVx)_h$ is the $h$-th entry of $UVx$, respectively.
Since one can easily show $\lVert z-UVx \rVert^2 = \lVert y-(UVx)^{\mathrm{y}} \rVert^2 + \lVert c-(UVx)^{\mathrm{c}} \rVert^2$,
the Multitask model $p_2(z|U,V,x)$ can be decomposed as
\begin{align}
p_2(z|U,V,x)=p_{22}(y|U,V,x)p_{21}(c|U,V,x),
\end{align}
where
\begin{align}
p_{22}(y|U,V,x) &\propto \exp\left( -\frac{1}{2}\lVert y - (UVx)^{\mathrm{y}} \rVert^2 \right), \\
p_{21}(c|U,V,x) &\propto \exp\left( -\frac{1}{2}\lVert c - (UVx)^{\mathrm{c}} \rVert^2 \right).
\end{align}
%In the same way as the expansion of the result for CBM,
Similar to the case of CBM, we define each distribution as follows:
\begin{align}
p_{22}^1(y|U,V,x) &\propto \exp\left( -\frac{1}{2}\lVert y -  (UVx)^{\mathrm{y}} \rVert^2 \right), \\
p_{22}^2(y|U,V,x) &= \prod_{j=1}^M (s_{M}((UVx)^{\mathrm{y}}))_j^{y_j}, \\
p_{21}^1(c|U,V,x) &\propto \exp\left( -\frac{1}{2}\lVert c - (UVx)^{\mathrm{c}} \rVert^2 \right), \\
p_{21}^2(c|U,V,x) &= \prod_{k=1}^K (\sigma_{K}((UVx)^{\mathrm{c}}))_k^{c_k} (1-((\sigma_{K}(UVx)^{\mathrm{c}}))_k)^{1-c_k}.
\end{align}
The data-generating distributions are denoted by
$q_{22}^1(y|x) = p_{22}^1(y|U_0,V_0,x)$, $q_{22}^2(y|x) = p_{22}^2(y|U_0,V_0,x)$,
$q_{21}^1(c|x) = p_{21}^1(c|U_0,V_0,x)$, and $q_{21}^2(c|x) = p_{11}^2(c|U_0,V_0,x)$, respectively.
Then Theorem \ref{thm-main-multitask} can be expanded as follows:
%Then, the expansion of Theorem 3.2 is as the following.

\begin{thm}[Expansion of Theorem \ref{thm-main-multitask}]
\label{thm-expand-multitask}
Same as Theorem \ref{thm-expand-cbm},
the models and data-generating distributions can be expressed as
\begin{align}
p_2^{ij}(z|U,V,x) &= p_{22}^i(y|U,V,x)p_{21}^j(c|U,V,x), \ i=1,2, \ j=1,2, \\
q_2^{ij}(z|x) &= p_2^{ij}(y,c|U_0,V_0,x), \ i=1,2, \ j=1,2.
\end{align}
Further, the KL divergences can be expressed as
\begin{align}
K_2^{ij}(U,V)=\displaystyle\iiint dydcdx q'(x)q_2^{ij}(y,c|x) \log \frac{q_2^{ij}(y,c|x)}{p_2^{ij}(y,c|U,V,x)}, \ i=1,2, \ j=1,2,
\end{align}
where $q'(x)$ is the input-generating distribution (as same as Definition \ref{def-MultitaskRLCT}).
Assume that $\mathscr{X}$ is positive definite and $(U,V)$ is in a compact set.
$\lambda_2$ and $m_2$ denote the functions of $(N,H,M,K,H_0)$ in Theorem \ref{thm-main-multitask}.
Then, the maximum pole $(-\lambda_2^{ij})$ and its order $m_2^{ij}$ of the zeta function can be written as
\begin{align}
\zeta_2^{ij}(x) = \iint K_2^{ij}(U,V)^z dUdV,
\end{align}
%are equal when $(i,j)=(1,1), (1,2), (2,1), (2,2)$.
for $i=1,2$ and $j=1,2$, we have
\begin{align}
\lambda_2^{1j} &= \lambda_2(N,H,M,K,H_0), \\
\lambda_2^{2j} &= \lambda_2(N,H,M-1,K,H_0), \\
m_2^{1j} &= m_2(N,H,M,K,H_0), \\
m_2^{2j} &= m_2(N,H,M-1,K,H_0).
\end{align}
\end{thm}
We prove Theorems \ref{thm-expand-cbm} and \ref{thm-expand-multitask} in \ref{sec:Proof}.
In addition, the above expanded theorems lead the following corollaries
that consider the case (the composed case) the outputs or concepts are composed of both real numbers and categorical variables.
Let $y^{\mathrm{r}}$ and $y^{\mathrm{c}}$ be the $M^{\mathrm{r}}$-dimensional real vector and $M^{\mathrm{c}}$-dimensional categorical variable, respectively.
These are observed variables of outputs.
Let $c^{\mathrm{r}}$ and $c^{\mathrm{c}}$ be the $K^{\mathrm{r}}$-dimensional real vector and $K^{\mathrm{c}}$-dimensional categorical variable, respectively.
They serve as concepts that describe the outputs from $N$-dimensional inputs.
In the same way of the definition of $z=[y;c]$, put $y=[y^{\mathrm{r}}; y^{\mathrm{c}}]$ and $c=[c^{\mathrm{r}}; c^{\mathrm{c}}]$.
Also, set $M=M^{\mathrm{r}}+M^{\mathrm{c}}$ and $K=K^{\mathrm{r}}+K^{\mathrm{c}}$,
where $M^{\mathrm{r}}, M^{\mathrm{c}}, K^{\mathrm{r}}, K^{\mathrm{c}} \geqq 1$.
Similarly, we have $ABx=[(ABx)^{\mathrm{r}}; (ABx)^{\mathrm{c}}]$ and $Bx=[(Bx)^{\mathrm{r}}; (Bx)^{\mathrm{c}}]$,
where
\begin{align}
(ABx)^{\mathrm{r}} &= ((ABx)_h)_{h=1}^{M^{\mathrm{r} }}, \
(ABx)^{\mathrm{c}} = ((ABx)_h)_{h=M^{\mathrm{r}}+1}^{M}, \\
(Bx)^{\mathrm{r}} &= ((Bx)_h)_{h=1}^{K^{\mathrm{r} }}, \
(Bx)^{\mathrm{c}} = ((Bx)_h)_{h=K^{\mathrm{r}}+1}^{K},
\end{align}
and $(ABx)_h$ and $(Bx)_h$ is the $h$-th entry of them, respectively.
Even if the outputs and concepts are composed of both real numbers and categorical variables,
using Theorems \ref{thm-expand-cbm} and \ref{thm-expand-multitask},
we can immediately derive the RLCT $\lambda_1^{\mathrm{com}}$ and its multiplicity $m_1^{\mathrm{com}}$ as follows:
%show that the RLCT of CBM is also equal to that derived in Theorem \ref{thm-main-cbm}.

\begin{cor}[RLCT of CBM in Composed Case]
\label{cor-expand-main}
Let $p_1^{\mathrm{com}}(y,c|A,B,x)$ be the statistical model of CBM in the composed case and $p_{12}^{\mathrm{com}}(y|A,B,x)$ and $p_{11}^{\mathrm{com}}(c|B,x)$ be the following probability distributions:
\begin{align}
p_{12}^{\mathrm{com}}(y|A,B,x) &\propto \exp\left( -\frac{1}{2}\lVert y^{\mathrm{r}} - (ABx)^{\mathrm{r}} \rVert^2 \right)
\times \prod_{j=1}^{M^{\mathrm{c}}} (s_{M^{\mathrm{c}}}((ABx)^{\mathrm{c}}))_j^{y_j}, \\
p_{11}^{\mathrm{com}}(c|B,x) &\propto \exp\left( -\frac{\gamma}{2}\lVert c^{\mathrm{r}} - (Bx)^{\mathrm{r}} \rVert^2 \right)
\times \left( \prod_{k=1}^{K^{\mathrm{c}}} (\sigma_{K^\mathrm{c}}((Bx)^{\mathrm{c}}))_k^{c_k} (1-(\sigma_{K^\mathrm{c}}((Bx)^{\mathrm{c}}))_k)^{1-c_k} \right)^{\gamma}.
\end{align}
The data-generating distribution is denoted by
\begin{align}
q_1^{\mathrm{com}}(y,c|x)=p_{12}^{\mathrm{com}}(y|A_0,B_0,x)p_{11}^{\mathrm{com}}(c|B_0,x).
\end{align}
The KL divergence can be expressed as
\begin{align}
K_1^{\mathrm{com}}(A,B) = \displaystyle\iiint dydcdx q'(x)q_1^{\mathrm{com}}(y,c|x) \log \frac{q_1^{\mathrm{com}}(y,c|x)}{p_1^{\mathrm{com}}(y,c|A,B,x)},
\end{align}
where $q'(x)$ is the input-generating distribution (as same as Definition \ref{def-CBMRLCT}).
Assume $\mathscr{X}$ is positive definite and $(A,B)$ is in a compact set
Then, the RLCT $\lambda_1^{\mathrm{com}}$ and its multiplicity $m_1^{\mathrm{com}}$ of $K_1^{\mathrm{com}}$ can be expressed as follows:
\begin{align}
\lambda_1^{\mathrm{com}} &= \frac{1}{2}(M^{\mathrm{r}}+M^{\mathrm{c}}+N-1)(K^{\mathrm{r}}+K^{\mathrm{c}}), \\
m_1^{\mathrm{com}} &= 1.
\end{align}
\end{cor}
This is because those concepts are decomposed as the real number part and the categorical one in the same way of correspondence between $z=[y; c]$ and $UVx=[(UVx)^{\mathrm{y}}; (UVx)^{\mathrm{c}}]$.
The composed case for Multitask is also easily determined as the following.
Note that $(UVx)^{\mathrm{r}}$ and $(UVx)^{\mathrm{c}}$ are defined in the same way of $ABx$.
\begin{cor}[RLCT of Multitask in Composed Case]
\label{cor-expand-main}
Let $p_2^{\mathrm{com}}(y,c|U,V,x)$ be the statistical model of Multitask in the composed case and $p_{22}^{\mathrm{com}}(y|U,V,x)$ and $p_{21}^{\mathrm{com}}(c|U,V,x)$ be
the following probability distributions:
\begin{align}
p_{22}^{\mathrm{com}}(y|U,V,x) &\propto \exp\left( -\frac{1}{2}\lVert y^{\mathrm{r}} - (UVx)^{\mathrm{r}} \rVert^2 \right) \times \prod_{j=1}^{M^{\mathrm{c}}} (s_{M^{\mathrm{c}}}((UVx)^{\mathrm{c}}))_j^{y_j}, \\
p_{21}^{\mathrm{com}}(c|U,V,x) &\propto \exp\left( -\frac{\gamma}{2}\lVert c^{\mathrm{r}} - (UVx)^{\mathrm{r}} \rVert^2 \right) \times \prod_{k=1}^{K^{\mathrm{c}}} (\sigma_{K^\mathrm{c}}((UVx)^{\mathrm{c}}))_k^{c_k} (1-(\sigma_{K^\mathrm{c}}((UVx)^{\mathrm{c}}))_k)^{1-c_k}.
\end{align}
The data-generating distribution is denoted by
\begin{align}
q_2^{\mathrm{com}}(y,c|x)=p_{22}^{\mathrm{com}}(y|U_0,V_0,x)p_{21}^{\mathrm{com}}(c|U_0,V_0,x).
\end{align}
Put the KL divergence as
\begin{align}
K_2^{\mathrm{com}}(U,V) = \displaystyle\iiint dydcdx q'(x)q_2^{\mathrm{com}}(y,c|x) \log \frac{q_2^{\mathrm{com}}(y,c|x)}{p_2^{\mathrm{com}}(y,c|U,V,x)},
\end{align}
where $q'(x)$ is the input-generating distribution (as same as Definition \ref{def-MultitaskRLCT}).
Assume $\mathscr{X}$ is positive definite and $(U,V)$ is in a compact set.
Then, the RLCT $\lambda_2^{\mathrm{com}}$ and its multiplicity $m_2^{\mathrm{com}}$ of $K_2^{\mathrm{com}}$ are as the followings:
\begin{align}
\lambda_2^{\mathrm{com}} &= \lambda_2(N,H,M^{\mathrm{r}}+M^{\mathrm{c}}-1,K^{\mathrm{r}}+K^{\mathrm{c}},H_0), \\
m_2^{\mathrm{com}} &= m_2(N,H,M^{\mathrm{r}}+M^{\mathrm{c}}-1,K^{\mathrm{r}}+K^{\mathrm{c}},H_0).
\end{align}
\end{cor}

\section{Discussion}
\label{sec:Discuss}
In this paper, we described how the RLCTs of CBM and Multitask can be determined in the case of a three-layered linear neural network.
Using these RLCTs and Eqs. (\ref{thm-watanabeG}) and (\ref{thm-watanabeF}),
we also clarified the exact asymptotic forms of the Bayesian generalization error and the marginal likelihood in these models.

%何をディスカッションしようか
%%CBMとマルチタスクの比較（HとH_0固定してKを動かすやつ、KとH_0固定してH動かすやつ）。
%%%%前者はコンセプト数に対する挙動、後者はマルチタスクの設計論的な観点？
%%出力やconceptのバイナリ化
%%余裕があればデータ発生させる数値実験
%%今後の展望（多層化、非線形化）は普通に→結論セクション行き

There are two limitations to this study.
The first is that this article treats three-layered neural networks.
If the input is an intermediate layer of a high accuracy neural network,
this model freezes when learning the last full-connected linear layer.
Thus, our result is valuable for the foundation of not only learning three-layered neural networks but also for transfer learning.
In fact, from the perspective of feature extracting, instead of the original input,
an intermediate layer of a state-of-the-art neural network can be used as an input to another model \cite{Donahue2014decaf,Sharif2014cnn,Yosinski2014transferable}.

The second limitation is that our formulation of CBM for Bayesian inference is based on Joint CBM.
There are two other types of CBM: Independent CBM and Sequential CBM \cite{Koh2020CBM20a}.
In Independent CBM, functions $x \mapsto c$ and $c \mapsto y$ are independently learned.
When the neural network is three-layered and linear, learning Independent CBM is equivalent to just estimating two independent linear transformation $c=Bx$ and $y=Ac$.
The graphical model of Independent CBM is $x \to B \to c \to A \to y$.
Clearly, $(A,B)$ is identifiable and the model is regular.
In contrast, Sequential CBM performs a two-step estimation.
First, $B$ is estimated as $c=Bx$.
Then, $A$ is learned as $y=A\hat{c}$,
where $\hat{B}$ is the estimator of $B$ and $\hat{c}=\hat{B}x$.
Since $\hat{c}$ is subject to a predictive distribution of $c$ conditioned $x$, its graphical model is the same as Joint CBM (Figure \ref{fig:CBM}).
Aiming the point of two-step estimation, for Bayesian inference of $A$, we set a prior of $B$ as the posterior of $B$ inferred by $c=Bx$, i.e. the prior distribution depends on the data.
If we ignored the point of two-step estimation, Bayesian inference of Sequential CBM would be that of Joint CBM.
Singular learning theory with data-dependent prior distribution is challenging because that theory use the prior as a measure of an integral to characterize the RLCT and its multiplicity (see Proposition \ref{prop-volume-dim}).
To resolve this issue, a new analysis method for Bayesian generalization error and free energy must be established.

%We can refer CBM to a parameter-restricted model.
Despite some of the above-mentioned limitations, through the contribution of this study,
we can obtain a new perspective that CBM is a parameter-restricted model.
According to the proof of Theorem \ref{thm-main-cbm}, the concept bottleneck structure $p_{11}(c|B,x)$ makes the neural network regular
whereas Standard is singular.
In other words, the concept bottleneck structure gives the constrain condition $B=B_0$ for the analytic set $K_1^{-1}(0)$ which we should consider for finding the RLCT.
This structure is added to Standard for interpretability.
Hence, in singular learning theory of interpretable models, Theorem \ref{thm-main-cbm} presents nontrivial results
which the constrain of the parameter for explanation affects the behavior of generalization;
this is the case in which the restriction for interpretability changes the model from singular to regular.
%%20230310ここまで校正

Finally, we discuss the model selection process for CBM and Multitask.
Bothe these models use a similar dataset composed of inputs, concepts, and outputs.
Additionally, they interpret the reason behind the predicted result using the observed concepts.
In both these approaches, supervised learning is carried out from the inputs to the concepts and outputs.
However, their model structures are different since CBM uses concepts for the middle layer units and Multitask does them to the outputs.
How does this difference affect generalization performance and accuracy of knowledge discovery?
We figure out that issue in the sense of Bayesian generalization error and free energy (negative log marginal likelihood).
Figures \ref{visrlct1K}-\ref{visrlct6K} show the behaviors of the RLCTs in CBM and Multitask when the number of concepts, i.e. $K$, increases.
In addition, Figures \ref{visrlct1H}-\ref{visrlct6H} illustrate the instances when the number of intermediate layer units $H$ increases.
%For CBM, the number of intermediate layer units is equal to that of concepts $K$.
%In contrast, one for Multitask is $H$ and independent to $K$.
In both these figures, the RLCT of CBM is a straight line and that of Multitask is a contour (piecewise linear).
For CBM, as mentioned in Definition \ref{def-CBMNN}, it is characterized by $(M,N,K)$ and the intermediate layer units is same as that of concepts in the network architecture.
As an inevitable conclusion, the RLCT of CBM does not depend on $H$ even if it uses the same pair $(y,c,x)$ as Multitask.
For Multitask, according to \cite{Aoyagi1}, the RLCT of Standard is also similar contour as a graph of a function from $H$ to the RLCT.
This similarity is immediately derived from Theorem \ref{thm-main-multitask} and \cite{Aoyagi1} (see also the proof of Theorem \ref{thm-main-multitask}).
Furthermore, we can determine the cross point between the RLCT curves of CBM and Multitask.
The RLCT dominates the asymptotic forms of the Bayesian generalization error and the free energy.
If the RLCT is greater, the Bayesian generalization error and the free energy also increases.
Thus, if their theoretical behaviors are clarified, %then we will be able to answer issues to select the method for data analysis.
then issues pertaining to the selection of the data analysis method can be clarified.
Hence, it is important for researchers and practitioners, for
whom accuracy is paramount, to compare CBM and Multitask.
Proposition \ref{prop-compare} gives us
%Hence, it is important for researchers and practitioners to
\begin{figure}[H]
 \begin{minipage}[b]{0.32\linewidth}
  \centering
  \includegraphics[width=5.3cm,height=3cm]
  {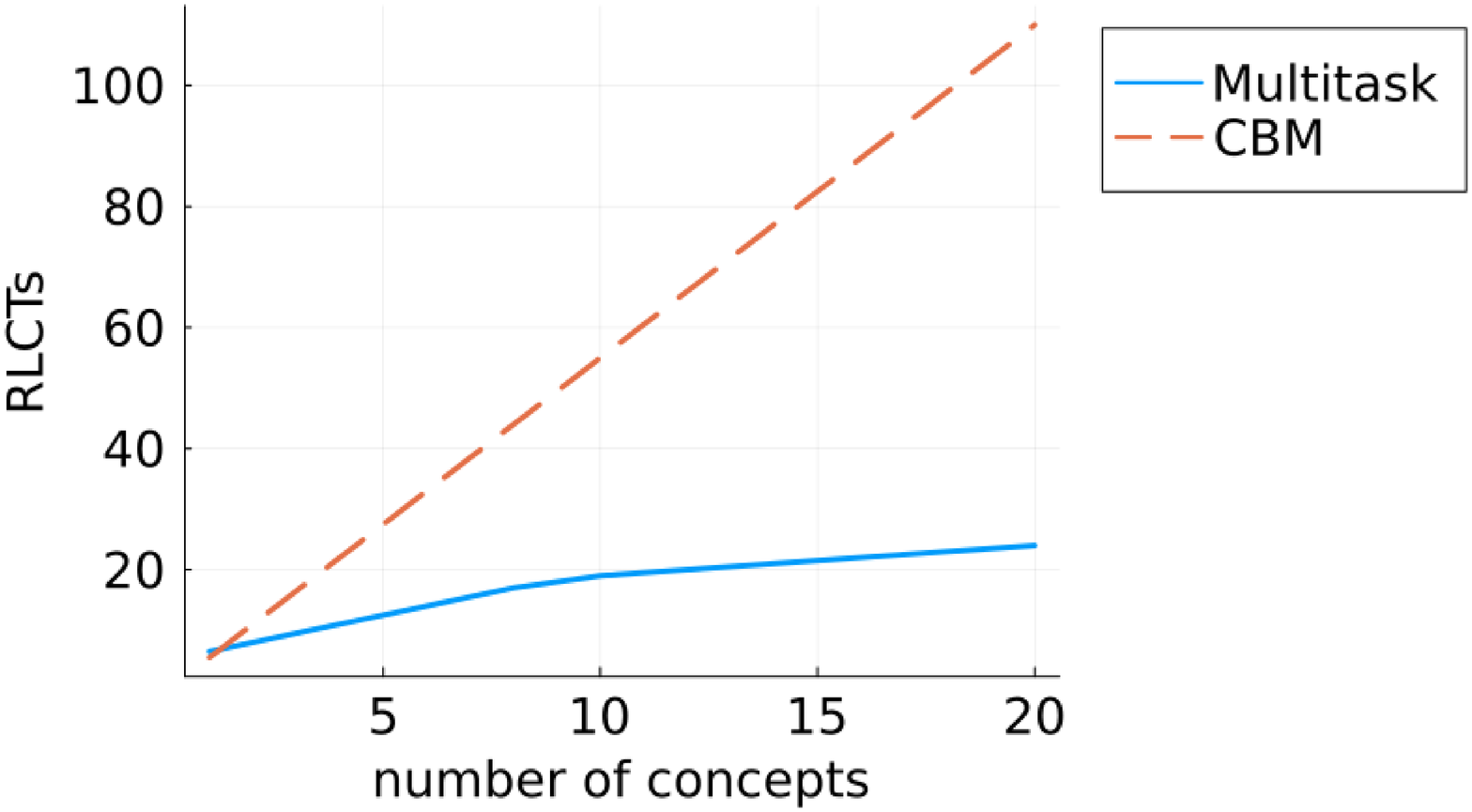}
  \subcaption{$H=3$, $H_0=1$}\label{visrlct1K}
 \end{minipage}
 \begin{minipage}[b]{0.32\linewidth}
  \centering
  \includegraphics[width=5.3cm,height=3cm]
  {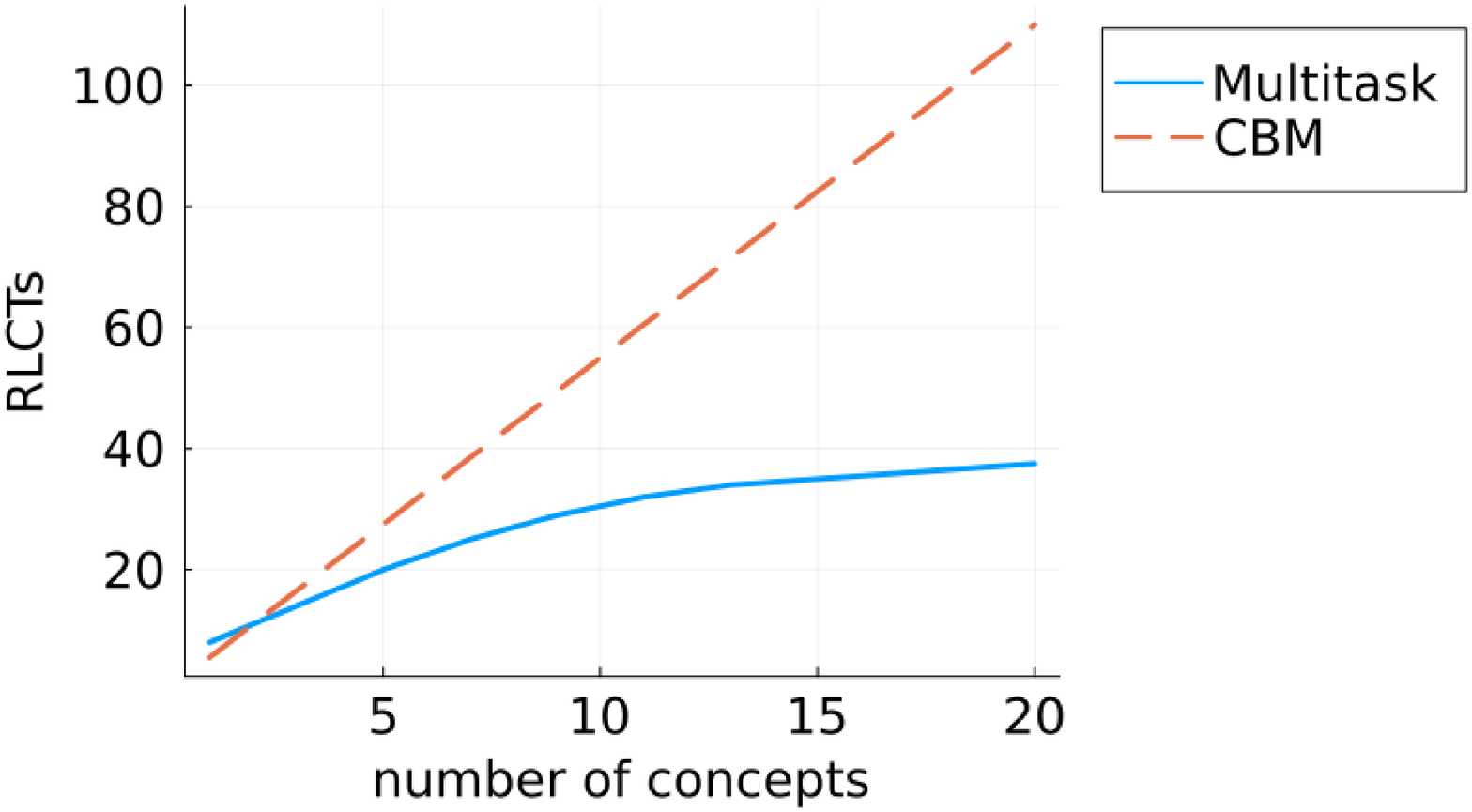}
  \subcaption{$H=6$, $H_0=1$}\label{visrlct2K}
 \end{minipage}
 \begin{minipage}[b]{0.32\linewidth}
  \centering
  \includegraphics[width=5.3cm,height=3cm]
  {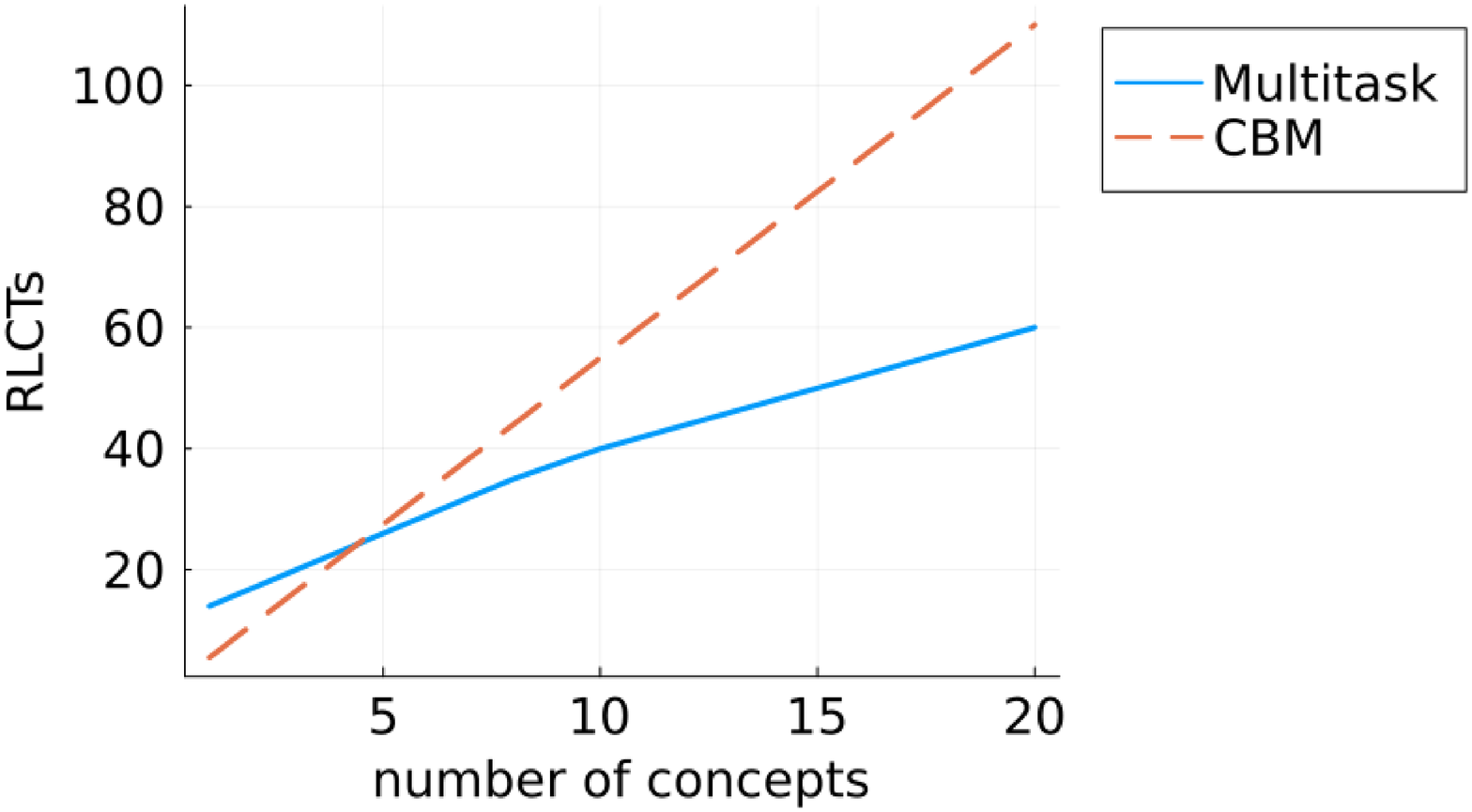}
  \subcaption{$H=6$, $H_0=4$}\label{visrlct3K}
 \end{minipage}\\
 \begin{minipage}[b]{0.32\linewidth}
  \centering
  \includegraphics[width=5.3cm,height=3cm]
  {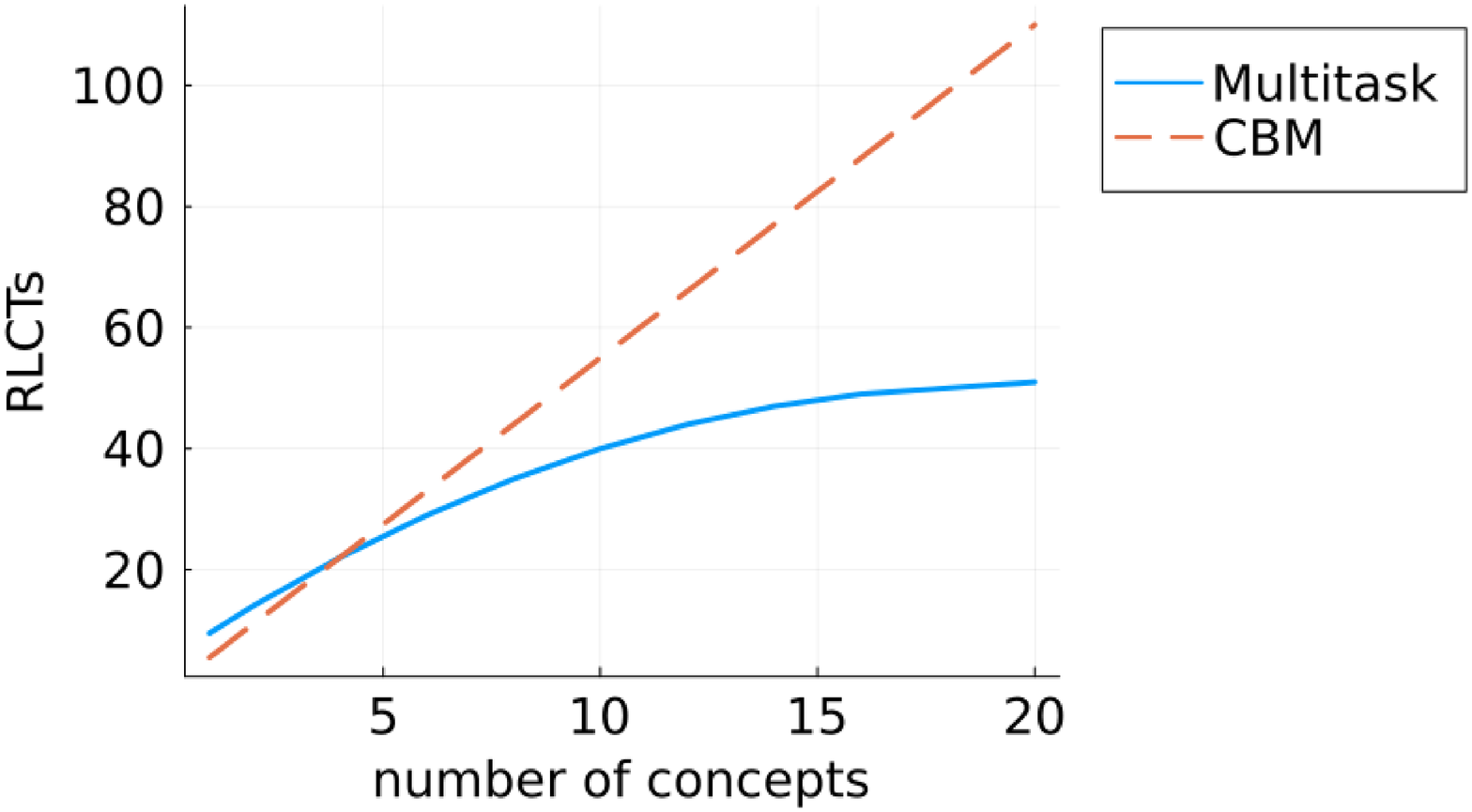}
  \subcaption{$H=9$, $H_0=1$}\label{visrlct4K}
 \end{minipage}
 \begin{minipage}[b]{0.32\linewidth}
  \centering
  \includegraphics[width=5.3cm,height=3cm]
  {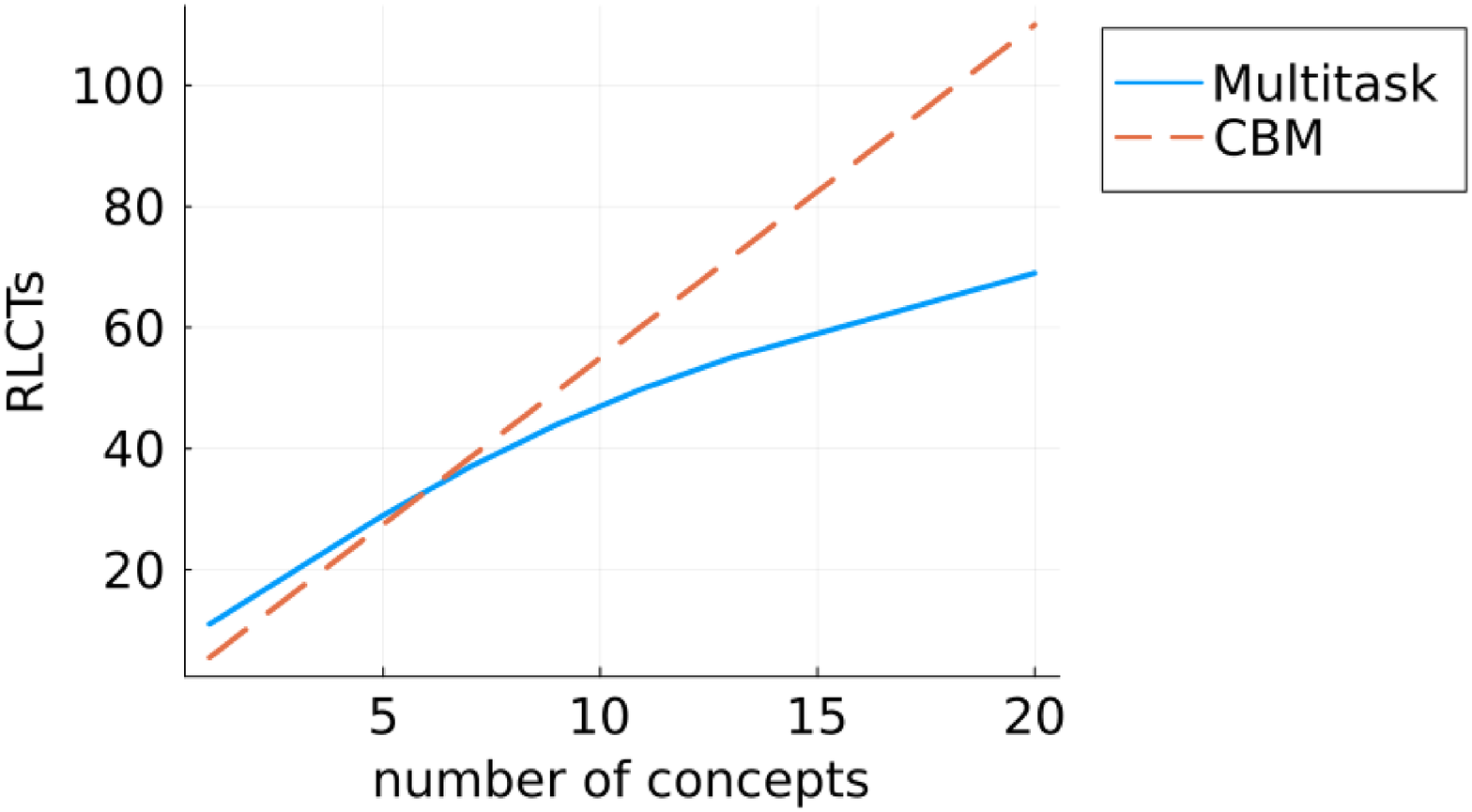}
  \subcaption{$H=9$, $H_0=4$}\label{visrlct5K}
 \end{minipage}
 \begin{minipage}[b]{0.32\linewidth}
  \centering
  \includegraphics[width=5.3cm,height=3cm]
  {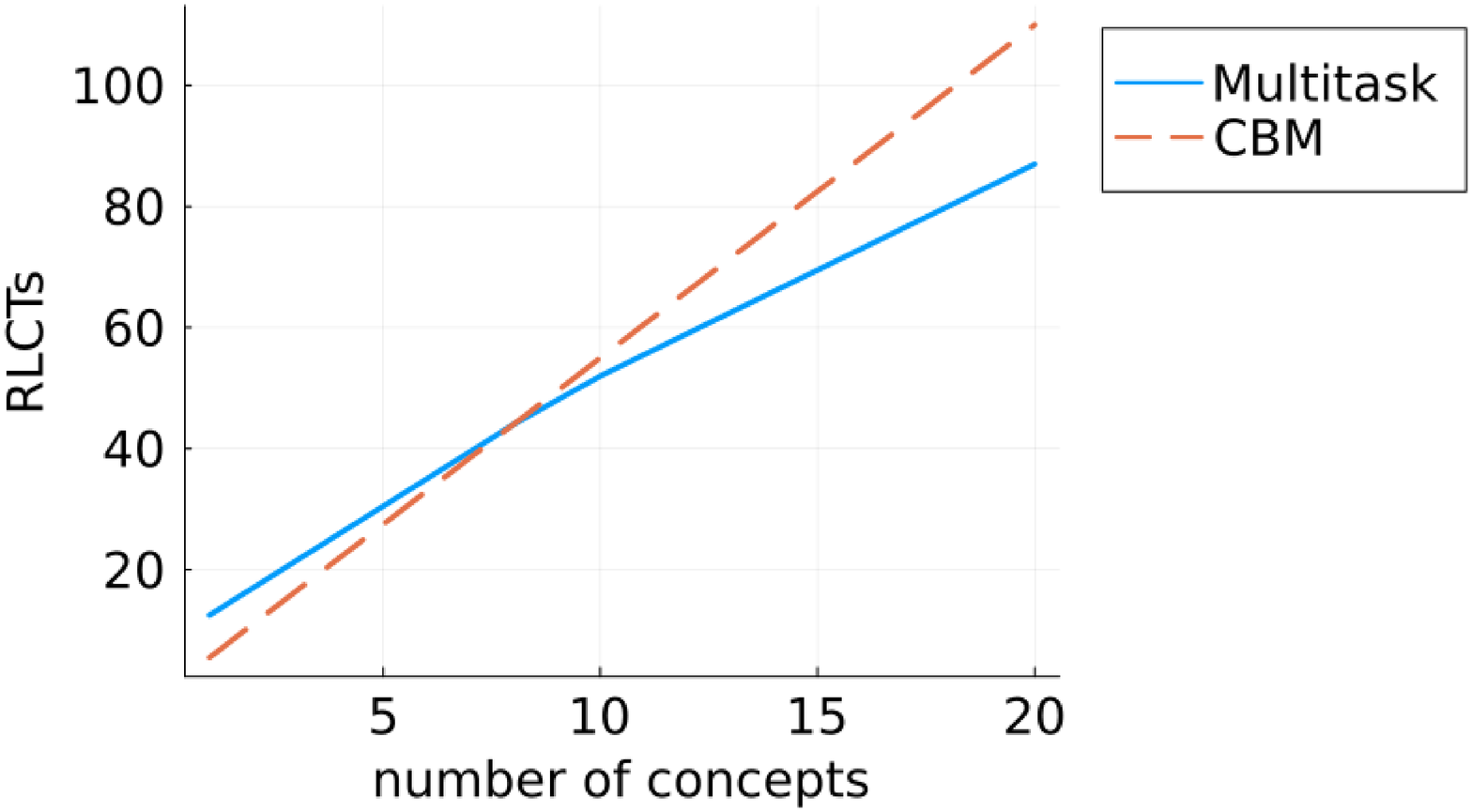}
  \subcaption{$H=9$, $H_0=7$}\label{visrlct6K}
 \end{minipage}
 \caption{Behaviors of RLCTs of CBM and Multitask as controlled by $K$.
The vertical axis represents the value of the RLCT and the horizontal one does the number of concepts $K$.
The RLCT behaviors are visualized as graphs of functions of $K$,
where $M=10$ and $N=1$ are fixed and $H$ and $H_0$ are set as the subcaptions.
The RLCT of CBM is drawn as dashed lines and that of Multitask as solid lines.
They are significantly different since one is linear and the other is non-linear (piecewise linear).
This is because the RLCT of Multitask is dependent of $H$ and $H_0$ but CBM is not.
}\label{visrlctK}
\end{figure}\!\!\!\!
\begin{figure}[H]
 \begin{minipage}[b]{0.32\linewidth}
  \centering
  \includegraphics[width=5.3cm,height=3cm]
  {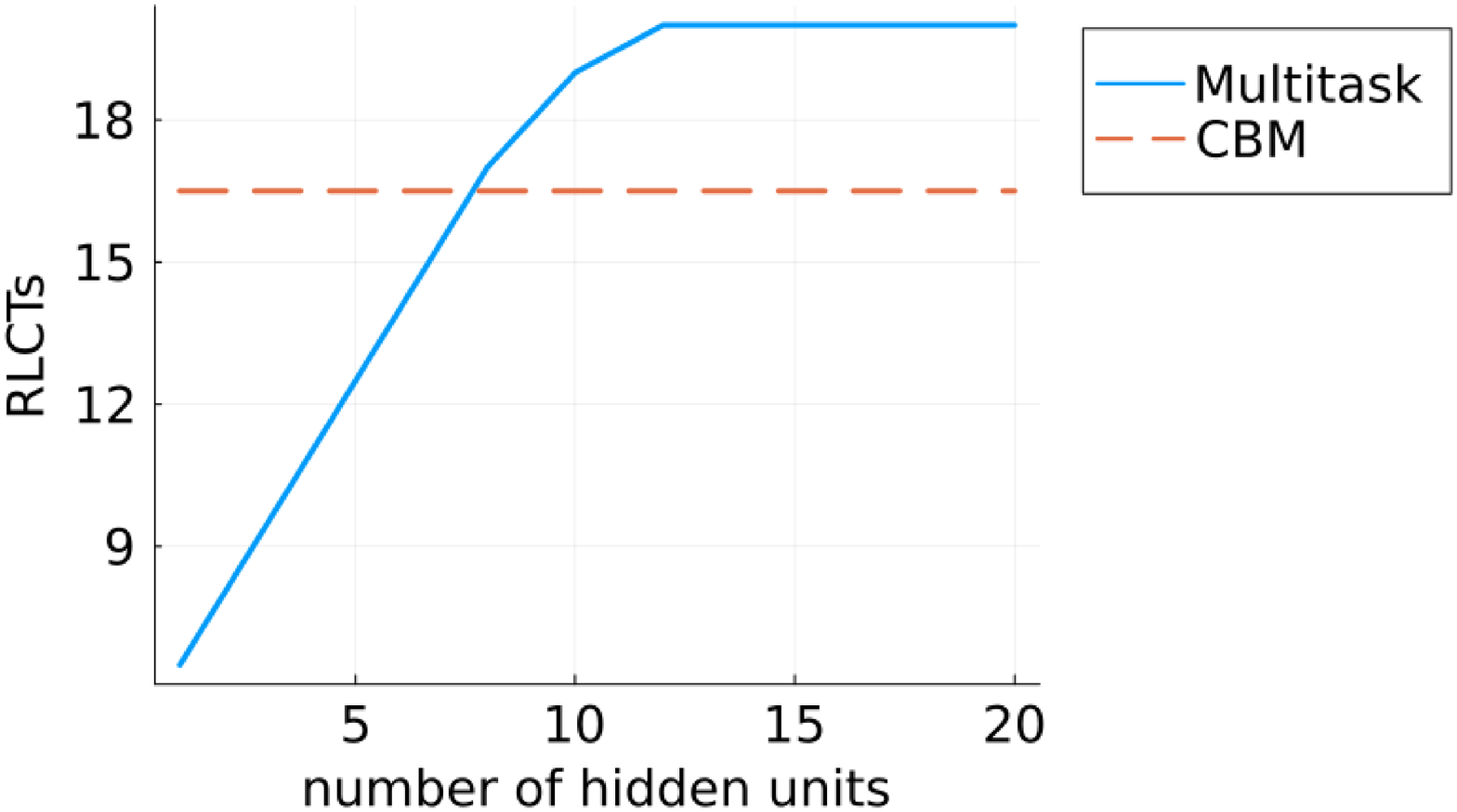}
  \subcaption{$K=3$, $H_0=1$}\label{visrlct1H}
 \end{minipage}
 \begin{minipage}[b]{0.32\linewidth}
  \centering
  \includegraphics[width=5.3cm,height=3cm]
  {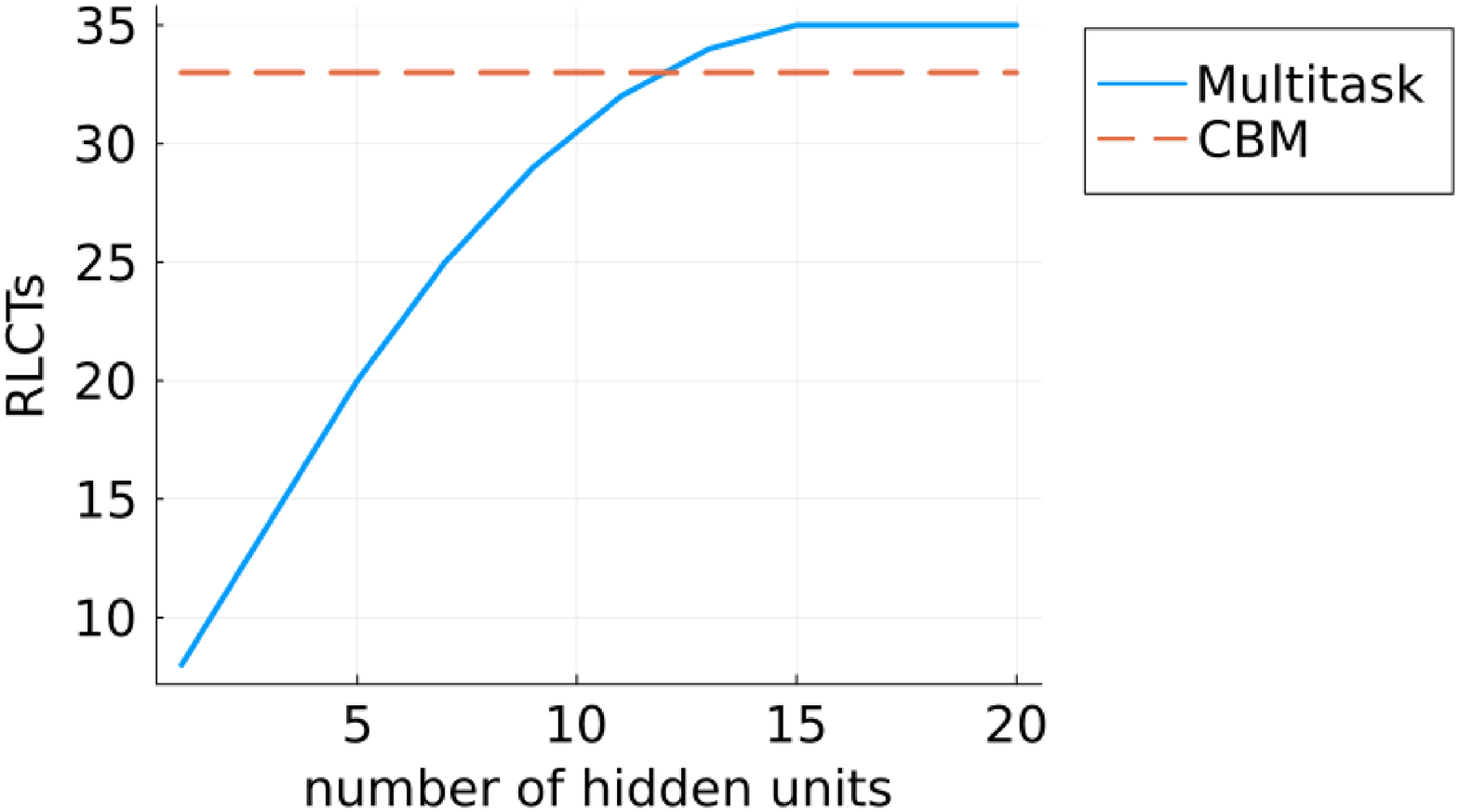}
  \subcaption{$K=6$, $H_0=1$}\label{visrlct2H}
 \end{minipage}
 \begin{minipage}[b]{0.32\linewidth}
  \centering
  \includegraphics[width=5.3cm,height=3cm]
  {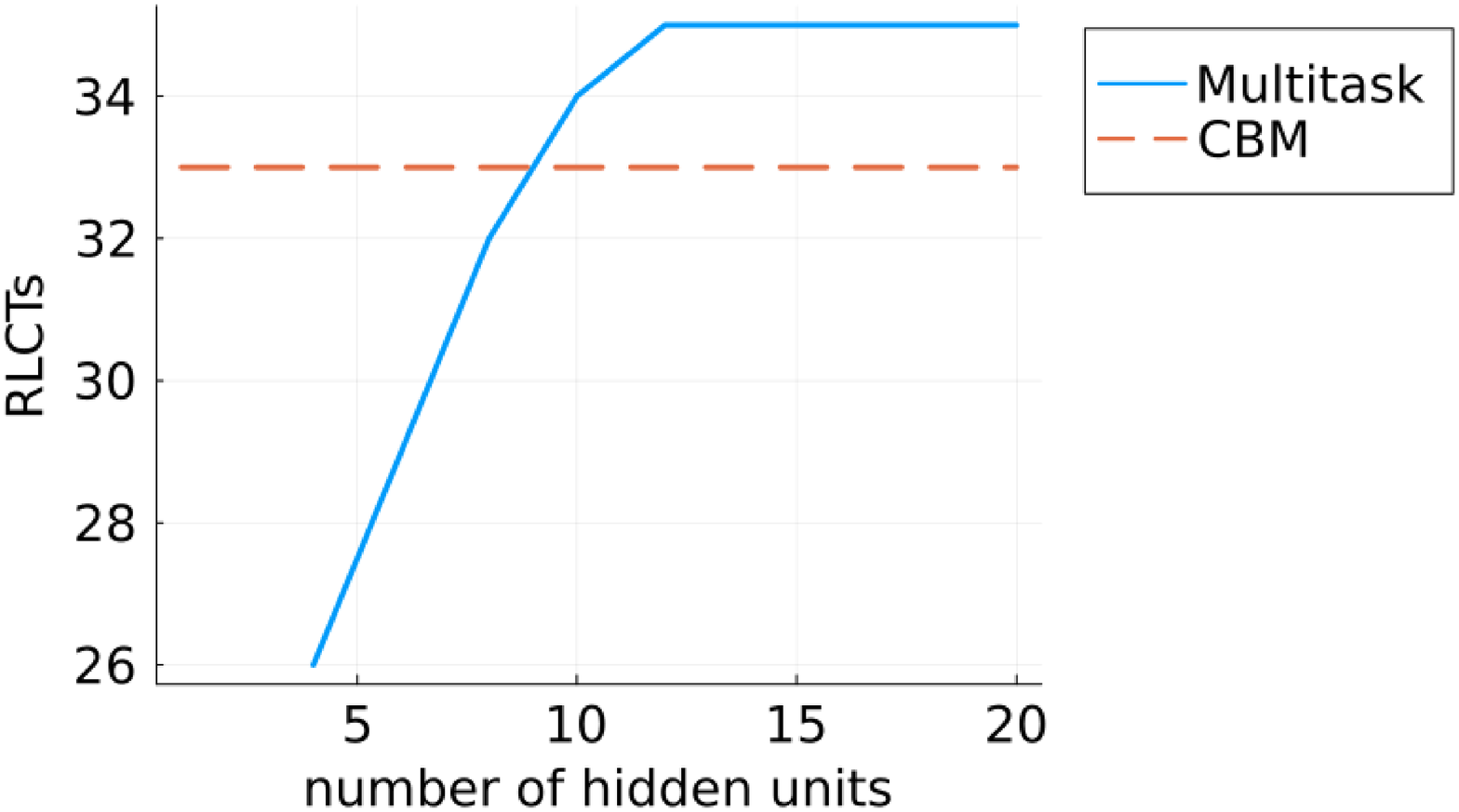}
  \subcaption{$K=6$, $H_0=4$}\label{visrlct3H}
 \end{minipage}\\
 \begin{minipage}[b]{0.32\linewidth}
  \centering
  \includegraphics[width=5.3cm,height=3cm]
  {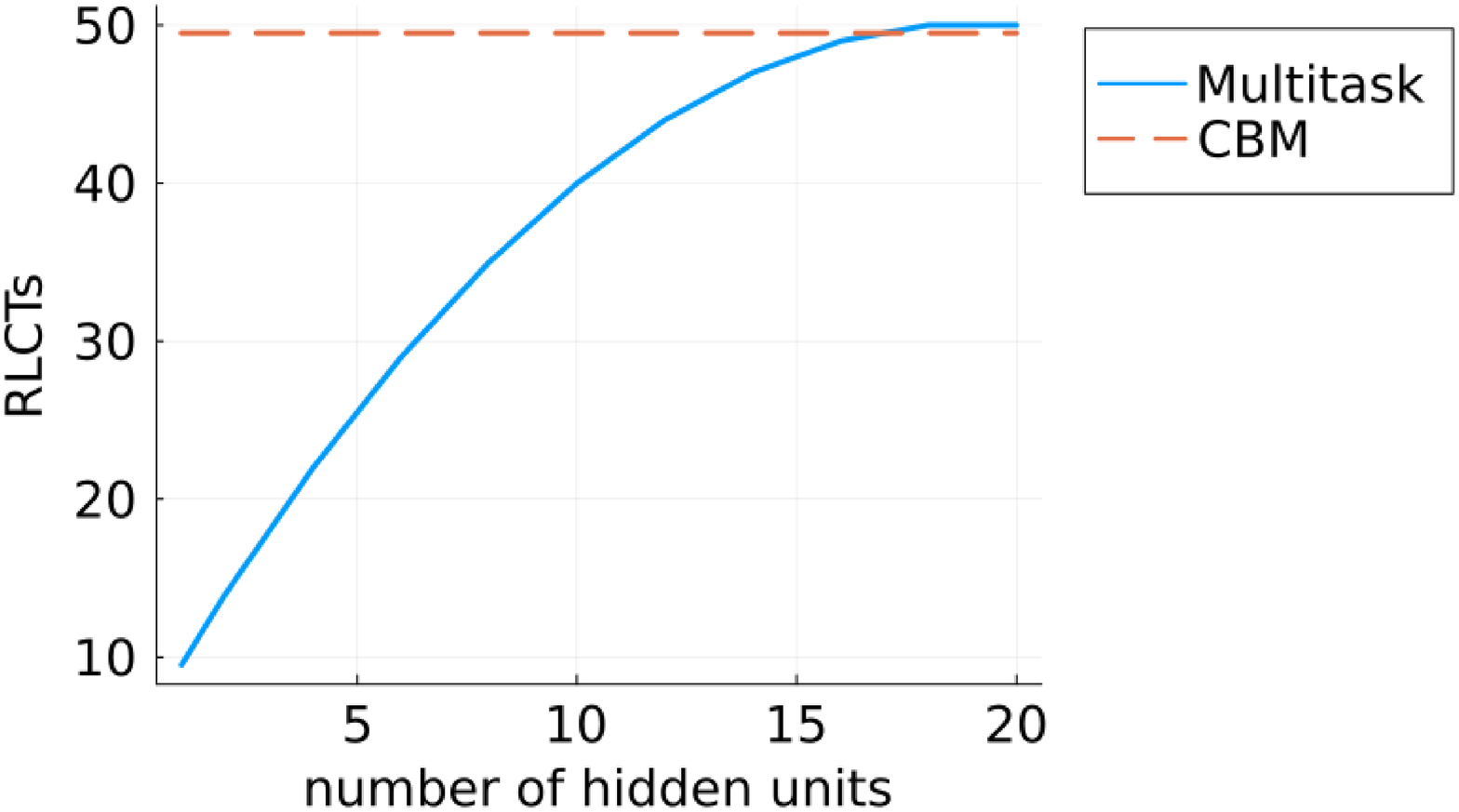}
  \subcaption{$K=9$, $H_0=1$}\label{visrlct4H}
 \end{minipage}
 \begin{minipage}[b]{0.32\linewidth}
  \centering
  \includegraphics[width=5.3cm,height=3cm]
  {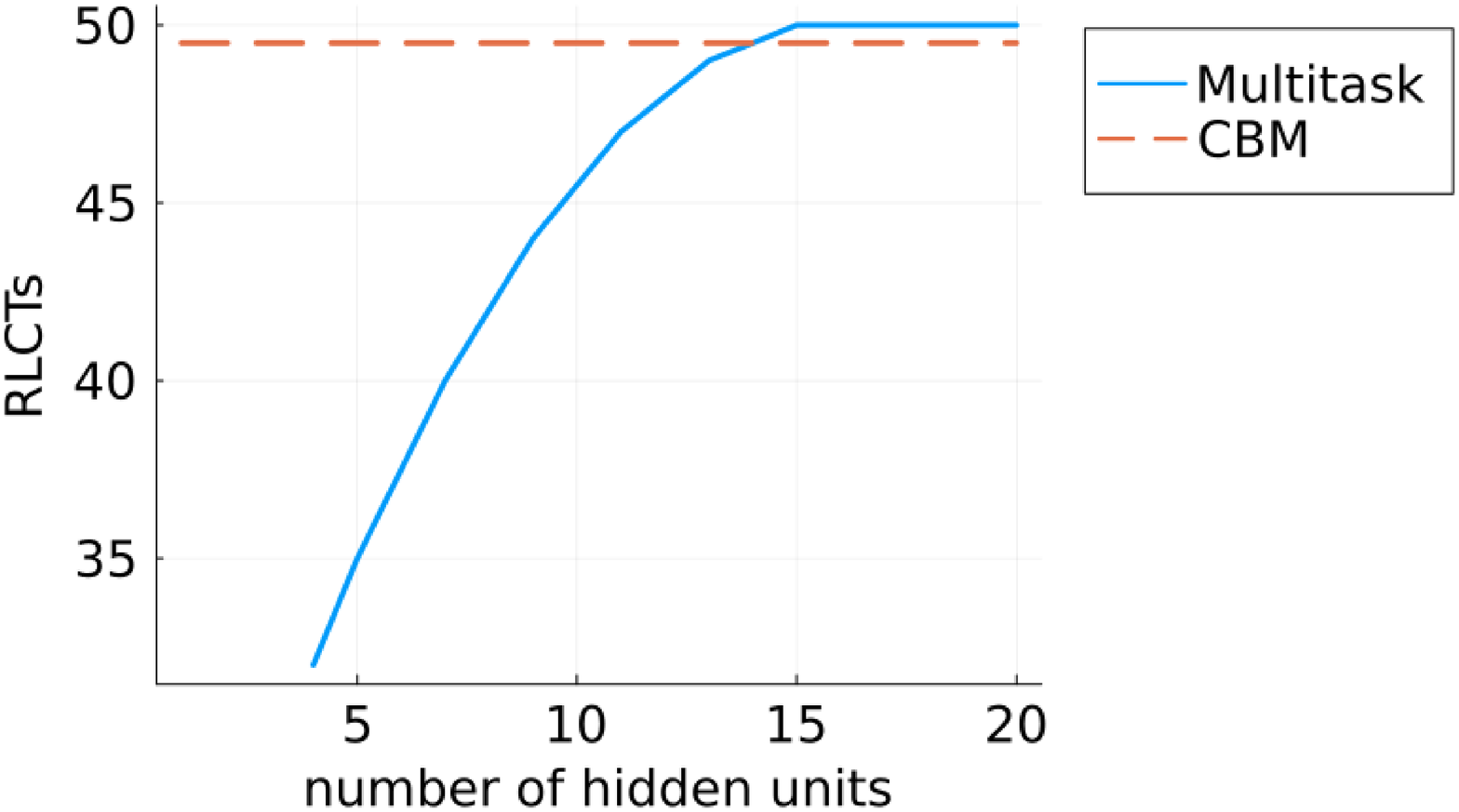}
  \subcaption{$K=9$, $H_0=4$}\label{visrlct5H}
 \end{minipage}
 \begin{minipage}[b]{0.32\linewidth}
  \centering
  \includegraphics[width=5.3cm,height=3cm]
  {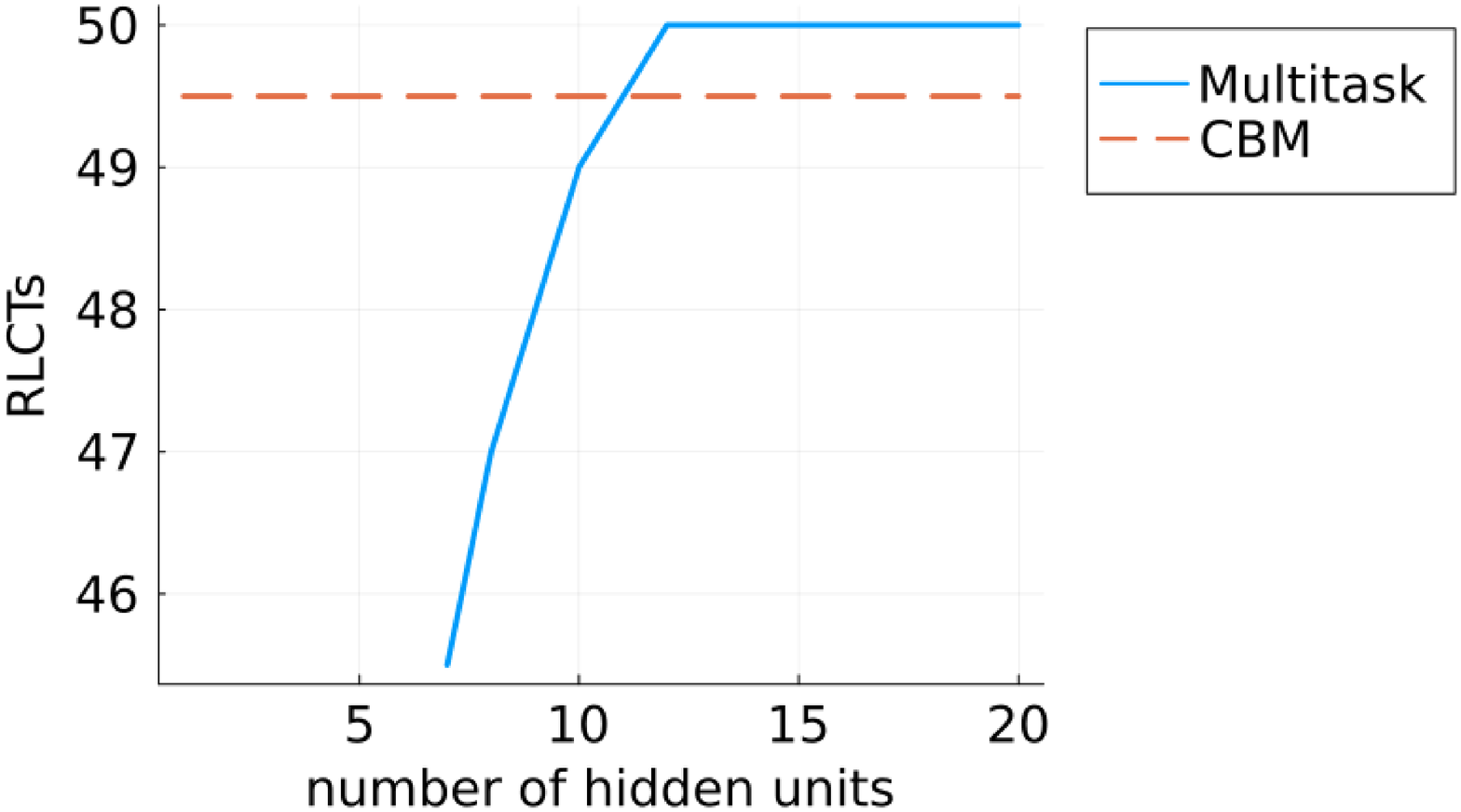}
  \subcaption{$K=9$, $H_0=7$}\label{visrlct6H}
 \end{minipage}
 \caption{Behaviors of RLCTs of CBM and Multitask as controlled by $H$.
The vertical axis represents the value of the RLCT and the horizontal one does the number of intermediate layer units (hidden units) $H$.
The behaviors of the RLCT are visualized as graphs of functions of $H$,
where $M=10$ and $N=1$ are fixed and $K$ and $H_0$ are set as the subcaptions.
The RLCT of CBM is drawn as dashed lines and that of Multitask is done as solid lines.
The RLCT of CBM does not depend on $H$; thus, it is a constant.
The RLCT of Multitask depends on $H$ as well as that of Standard as clarified in \cite{Aoyagi1}.}\label{visrlctH}
\end{figure}\clearpage \noindent
%whom accuracy is paramount, to compare CBM and Multitask.
%Proposition \ref{prop-compare} gives us
change points for an issue: which has better performance for given $(M,H,N,K)$ and assumed $H_0$.
If $\lambda_1 > \lambda_2$, Multitask has better performance than CBM in the sense of the Bayesian generalization error and the free energy.
If $\lambda_1 \leqq \lambda_2$, the opposite fact holds.
The proof of Proposition \ref{prop-compare} lies in \ref{sec:Proof}.

\begin{prop}[Comparison the RLCTs of CBM and Multitask]
\label{prop-compare}
Along with the conditional branch in Theorem \ref{thm-main-multitask},
the magnitude of the RLCT of CBM $\lambda_1$ and that of Multitask $\lambda_2$ changes as the following.
\begin{enumerate}
\item In the case $M+K+H_0 \leqq N+H$ and $N+H_0 \leqq M+K+H$ and $H+H_0 \leqq N+M+K$,
    \begin{enumerate}
        \item and if $N+M+K+H+H_0$ is even, then
        $$\begin{cases}
        \lambda_1 > \lambda_2 & (K > H+H_0 - (\sqrt{M}-\sqrt{N})^2), \\
        \lambda_1 \leqq \lambda_2 & ({\rm otherwise}).
        \end{cases}$$
        \item and if $N+M+K+H+H_0$ is odd, then
        $$\begin{cases}
        \lambda_1 > \lambda_2 & (K >  H+H_0-M-N + \sqrt{4MN+1}), \\
        \lambda_1 \leqq \lambda_2 & ({\it otherwise}).
        \end{cases}$$
    \end{enumerate}
\item In the case $N+H<M+K+H_0$, then
$$\begin{cases}
        \lambda_1 > \lambda_2 & ((M+N-H_0)K > (N-H_0)H+MH_0), \\
        \lambda_1 \leqq \lambda_2 & ({\it otherwise}).
        \end{cases}$$
\item In the case $M+K+H<N+H_0$, then
$$\begin{cases}
        \lambda_1 > \lambda_2 & ((M+N-H_0)K > (N-H)H_0+MH), \\
        \lambda_1 \leqq \lambda_2 & ({\it otherwise}).
        \end{cases}$$
\item Otherwise (i.e. $N+M+K<H+H_0$), then
$$\begin{cases}
        \lambda_1 > \lambda_2 & (K > N), \\
        \lambda_1 \leqq \lambda_2 & ({\it otherwise}).
        \end{cases}$$
\end{enumerate}
\end{prop}

\section{Conclusion}
\label{sec:Conclusion}
We obtain the exact asymptotic behaviors of Bayesian generalization and free energy in neural networks with a concept bottleneck model and multitask formulation when the networks are three-layered linear.
The behaviors are derived by finding the real log canonical thresholds of these models.
The results show that concept bottleneck structure makes the neural network regular (identifiable)  in the case of a three-layered and linear one.
On the other hand, multitask formulation for a three-layered linear network only involves the addition of the concepts to the output;
hence, the behaviors of the Bayesian generalization error and free energy are similar to that of the standard model.
A future work would involve the theoretical analysis for multilayer and non-linear activation.
Another would involve formulating Sequential CBM based on the singular learning theory.
Clarifying numerical behaviors of Main Theorems can be yet another future research direction.

%\section*{Acknowledgement}
%The authors appreciate the editors and the reviewers for improving our paper.
%WIPWIPWIP[ The authors appreciate ...... ]WIPWIPWIP.



\appendix

\section{Singular Learning Theory}
\label{sec:Singular}
%%これを主定理とベイズ推論説明の間に書くと本題までが遠くてつらい気持ちになる読者もいそうだという所感あり
%%付録にすべきか主定理の前にすべきかは澤田さんにも相談したい（部分的にイントロへ移動も検討に含む）
%\subsection{Relationship between Bayesian inference and Algebraic Geometry}
%%LDA論文の表現のままが多いので書き直しは必要: 20221215いくらか直した
We briefly explain the relationship between Bayesian inference and algebraic geometry:
in other words, the reason behind the need for resolution of singularity.
This theory is referred to as the singular learning theory \cite{SWatanabeBookE}.

It is useful for the following analytic form \cite{Atiyah1970resolution} of the singularities resolution theorem \cite{Hironaka} to treat $K(w)$ in Eq. (\ref{def-aveerror}) and its zero points $K^{-1}(0)$ in Eq. (\ref{algvariety}).
\begin{thm}[Hironaka, Atiyah]\label{thm-singular}
Let $K$ be a non-negative analytic function on $\mathcal{W} \subset \mathbb{R}^d$.
Assume that $K^{-1}(0)$ is not an empty set.
Then, there are an open set $\mathcal{W}'$, a $d$-dimensional smooth manifold $\mathcal{M}$, and an analytic map $g:\mathcal{M} \rightarrow \mathcal{W}'$ such that $g: \mathcal{M}\setminus g^{-1}(K^{-1}(0)) \rightarrow \mathcal{W}' \setminus K^{-1}(0)$ is isomorphic and
\begin{gather}
K(g(u))=u_1^{2k_1} \ldots u_d^{2k_d}, \\
|\det g'(u)| =b(u)|u_1^{h_1} \ldots u_d^{h_d}|
\end{gather}
hold for each local chart $U \ni u$ of $\mathcal{M}$,
where $k_j$ and $h_j$ are non-negative integer for $j=1,\ldots,d$, $\det g'(u)$ is the Jacobian of $g$ and $b:\mathcal{M} \rightarrow \mathbb{R}$ is strictly positive analytic: $b(u)>0$. 
\end{thm}
Atiyah has derived this form for analyzing the relationship between a division of distributions (a.k.a. hyperfunctions) and local type zeta functions \cite{Atiyah1970resolution}.
By using Theorem \ref{thm-singular}, the following is proved \cite{Atiyah1970resolution, Bernstein1972, Sato1974zeta}.
\begin{thm}[Atiyah, Bernstein, Sato and Shintani]\label{thm-zeta}
Let $K: \mathbb{R}^d \rightarrow \mathbb{R}$ be an analytic function of a variable $w \in \mathcal{W}$.
$a: \mathcal{W} \rightarrow \mathbb{R}$ is denoted by a $C^{\infty}$-function with compact support $\mathcal{W}$.
The following univariate complex function
\begin{align}
\zeta(z) = \int_\mathcal{W} |K(w)|^z a(w) dw
\end{align}
is a holomorphic function in $\mathrm{Re}(z)>0$.
Moreover, $\zeta(z)$ can be analytically continued to a unique meromorphic function on the entire complex plane $\mathbb{C}$.
Its all poles are negative rational numbers.
\end{thm}
\begin{comment}
Applying this theorem to the KL divergence Eq. (\ref{def-aveerror}), we have
\begin{gather}
K(g(u))=u_1^{2k_1} \ldots u_d^{2k_d}, \\
|\det g'(u)| =b(u)|u_1^{h_1} \ldots u_d^{h_d}|.
\end{gather}
\end{comment}
Suppose the prior density $\varphi(w)$ has the compact support $\mathcal{W}$ and the open set $\mathcal{W}'$ satisfies $\mathcal{W} \subset \mathcal{W}'$.
By using Theorem \ref{thm-zeta}, we can define a zeta function of learning theory.
\begin{defi}[Zeta Function of Learning Theory]
\label{def-learnzeta}
Let $K(w) \geqq 0$ be the KL divergence mentioned in Eq. (\ref{def-aveerror}) and $\varphi(w)\geqq 0$ be a prior density function which satisfies the above assumption. A zeta function of learning theory is defined by the following univariate complex function
$$\zeta(z) = \int_\mathcal{W} K(w)^z \varphi(w) dw.$$
\end{defi}
\begin{defi}[Real Log Canonical Threshold]
Let $\zeta(z)$ be a zeta function of learning theory represented in Definition \ref{def-learnzeta}.
Consider an analytic continuation of $\zeta(z)$ from Theorem \ref{thm-zeta}.
A real log canonical threshold (RLCT) $\lambda$ is defined by the negative maximum pole of $\zeta(z)$ and its multiplicity $m$ is defined by the order of the maximum pole:
\begin{gather}
\zeta(z) = \frac{C(z)}{(z+\lambda)^m}\frac{C_1(z)}{(z+\lambda_1)^{m_1}}\ldots\frac{C_D(z)}{(z+\lambda_D)^{m_D}}\ldots, \\
\lambda < \lambda_k \ (k=1,\ldots,D,\ldots),
\end{gather}
where $C(z)$ and $C_k(z)$ $(k=1,\ldots,D,\ldots)$ are non-zero-valued complex functions.
\end{defi}
Watanabe constructed the singular learning theory;
he proved that the RLCT $\lambda$ and the multiplicity $m$ determine the asymptotic Bayesian generalization error and free energy \cite{Watanabe1, Watanabe2, SWatanabeBookE}:
\begin{thm}[Watanabe]
$\zeta(z)$ is denoted by the zeta function of learning theory as Definition \ref{def-learnzeta}.
Let $\lambda$ and $m$ be the RLCT and the multiplicity defined by $\zeta(z)$.
The Bayesian generalization error $G_n$ and the free energy $F_n=-\log Z_n$ have the asymptotic forms $(\ref{thm-watanabeG})$ and $(\ref{thm-watanabeF})$ showen in section 1.
\end{thm}
This theorem is rooted in Theorem \ref{thm-singular}.
That is why we need resolution of singularity to clarify the behavior of $G_n$ and $F_n$ via determination of the RLCT and its multiplicity.

%\subsection{Theoretical Calculating RLCT}
Here, we describe how to determine the RLCT $\lambda>0$ of the model corresponding to $K(w)$.
We apply Theorem \ref{thm-singular} to the zeta function of learning theory.
Since we assumed the parameter space is compact, the manifold in singularity resolution is also compact.
Thus, the manifold can be covered by a union of $[0,1)^d$ for each local coordinate $U$.
Considering the partision of unity for $[0,1)^d$, we have
\begin{align}
\zeta(z) &= \int_U K(g(u))^z \varphi(g(u)) |\det g'(u)| du \\
&= \sum_{\eta} \int_U K(g(u))^z \varphi(g(u)) |\det g'(u)| \phi_{\eta}(u) du \\
&= \sum_{\eta} \int_{[0,1]^d} u_1^{2k_1z+h_1} \ldots u_d^{2k_dz+h_d}\varphi(g(u))b(u)\phi_{\eta}(u) du,
\end{align}
where $\phi_{\eta}$ is the partision of unity: $\mathrm{supp}(\phi_{\eta})=[0,1]^d$ and $\phi_{\eta}(u)>0$ in $(0,1)^d$.
The functions $\varphi(g(u))$, $b(u)$, and $\phi_{\eta}$ are strictly positive in $(0,1)^d$;
thus, we should consider the maximum pole of
\begin{equation}
\int_{[0,1]^d} u_1^{2k_1z+h_1} \ldots u_d^{2k_dz+h_d} du = \frac{1}{2k_1 z + h_1+1}\ldots\frac{1}{2k_d z + h_d+1}.
\end{equation}
%\begin{equation}
%\int_{[0,1]^d} u_1^{2k_1z+h_1} \ldots u_d^{2k_dz+h_d} du = \frac{C^U_1(z)}{2k_1 z + h_1+1}\ldots\frac{C^U_d(z)}{2k_d z + h_d+1},
%\end{equation}
%where $(C^U_j(z))_{j=1}^d$ are non-zero functions of $z \in \mathbb{C}$.
Allowing duplication, the set of the poles can be represented as follows:
\begin{equation}
\left\{ \frac{h_j + 1}{2k_j} \mid j=1,\ldots, d \right\}.
\end{equation}
Thus, we can find the maximum pole $(-\lambda_U)$ in the local chart $U$ as follows
\begin{equation}
\lambda_U = \min_{j=1}^d \left\{ \frac{h_j + 1}{2k_j} \right\}.
\end{equation}
By considering the duplication of indices, we can also find the multiplicity in $U$ denoted as $m_U$.
Therefore, we can determine the RLCT as $\lambda = \min_U \lambda_U$ and the multiplicity $m$ as the order of the pole $(-\lambda)$, i.e. $m=m_{\underline{U}}$ where $\underline{U} = \mathrm{argmin}_U \lambda_U$.

In addition, we explain a geometrical property of the RLCT as follows:
a limit of a volume dimension \cite{YamazakiPhDThesis, SWatanabeBookE}:
\begin{prop}\label{prop-volume-dim}
Let $V: (0, \infty) \to (0, \infty)$, $t \mapsto V(t)$ be a volume of $K^{-1}((0,t))$ measured by $\varphi(w)dw$:
\begin{align}
V(t) = \int_{K(w)<t}\varphi(w)dw.
\end{align}
Then, the RLCT $\lambda$ satisfies the following:
\begin{align}
\lambda = \lim_{t \to +0} \frac{\log V(t)}{\log t}.
\end{align}
\end{prop}
The RLCT and its multiplicity are birational invariants of an analytic set $K^{-1}(0)$.
Since they are birational invariants, they do not depend on the resolution of singularity.
The above property characterizes that fact.

To determine $\lambda$ and $m$, we should consider the resolution of singularity \cite{Hironaka} for concrete varieties corresponding to the models.
We should calculate theoretical values of RLCTs to a family of functions to clarify a learning coefficient of a singular statistical model; however,
there exists no standard method finding RLCTs to a given collection of functions.
Thus, we need different procedures for RLCT of each statistical model.
In fact, as mentioned in section \ref{sec:Intro}, RLCTs of several models has been analyzed in both statistics and machine learning fields for each cases.
\begin{comment}
For example, the RLCTs had been studied in the following:
mixture models \cite{Yamazaki1, Yamazaki2004BinMixRLCT, SatoK2019PMM, WatanabeT2022MultiMixRLCT},
Boltzmann machines \cite{Yamazaki4,Aoyagi2,Aoyagi3},
non-negative matrix factorization \cite{nhayashi2,nhayashi5,nhayashi8},
latent class analysis \cite{Drton2009LCARLCT},
latent Dirichlet allocation \cite{nhayashi7, nhayashi9},
naive Bayesian networks \cite{Rusakov2005asymptotic},
Bayesian networks \cite{Yamazaki3},
Markov models \cite{Zwiernik2011asymptotic},
hidden Markov models \cite{Yamazaki2},
linear dynamical systems \cite{Naito2014KFRLCT},
Gaussian latent tree and forest models \cite{Drton2017forest},
and three-layered neural networks whose activation function is linear \cite{Aoyagi1}, analytic-odd (like $\tanh$) \cite{Watanabe2}, and Swish \cite{Tanaka2020SwishNNRLCT}.
\end{comment}
%Note that clarifying the exact value of the RLCT in the all case is challenging problem.
%Whereas we would like to emphasize that this is not to deny the value and novelty of these researches.
%In deed, they cannot have clarified the exact value except for Aoyagi's result in 2005 \cite{Aoyagi1}.
Our work for CBM and Multitask contributes the body of knowledge in learning theory: clarifying RLCT of singular statistical model.
Value of such studies in the practical perspective is introduced in section \ref{sec:Intro}.

\section{Proofs of Claims}
\label{sec:Proof}
Let $\sim$ be a binomial relation whose both hand sides have the same RLCT and multiplicity,
and $\mathrm{M}(M,N)$ be the set of $M \times N$ real matrices.
Then, we define the following utility.
\begin{defi}[Rows Extractor]
Let $(\cdot)_{<d}: \mathrm{M}(I,J) \to \mathrm{M}(d-1,J)$, $2 \leqq d \leqq I$, $1 \leqq J$ be an operator for a matrix to extract the following submatrix:
\begin{align}
(W)_{<d} = \left(
\begin{matrix}
w_{11} & \ldots & w_{1J} \\
\vdots & \ddots & \vdots \\
w_{(d-1)1} & \ldots & w_{(d-1)J}
\end{matrix}
\right), \ W=(w_{ij})_{i=1, j=1}^{I,J}, \ W \in \mathrm{M}(I,J).
\end{align}
\end{defi}
We use this operator for a vector $w \in \mathbb{R}^M$ as $\mathbb{R}^M \cong \mathrm{M}(M,1)$; we refer to this as a column vector.
In the same way as above, we define $(\cdot)_{> d}$, $(\cdot)_{\leqq d}$, and $(\cdot)_{\geqq d}$, where inequalities correspond to the row index.
Also, $(\cdot)_{\ne d}$ and $(\cdot)_{= d}$ are defined as $[(\cdot)_{< d}; (\cdot)_{> d}]$ and $((\cdot)_{\leqq d})_{\geqq d}$, respectively.
We can immediately show $(WV)_{< d}=(W)_{< d}V$ for $W \in \mathrm{M}(I,R)$ and $V \in \mathrm{M}(R,J)$, where $R \geqq 1$.
In addition, other relations in the subscript satisfy the above rule; we can derive
$(WV)_{> d}=(W)_{> d}V$,
$(WV)_{\leqq d}=(W)_{\leqq d}V$,
$(WV)_{\geqq d}=(W)_{\geqq d}V$,
$(WV)_{\ne d}=(W)_{\ne d}V$, and $(WV)_{= d}=(W)_{= d}V$.
Moreover, to prove our theorems, we use the following lemmas \cite{SWatanabeBookE, Aoyagi1, Matsuda1-e}.

\begin{lem}
\label{lem-equivRLCT}
Let $f_1: \mathcal{W} \to \mathbb{R}$ and $f_2: \mathcal{W} \to \mathbb{R}$ be non-negative analytic functions.
If there are positive constants $\alpha_1, \alpha_2 >0$ such that
\begin{align}
\alpha_1f_1(w) \leqq f_2(w) \leqq \alpha_2 f_1(w)
\end{align}
on the neighborhood of $f_2^{-1}(0)$,
then $f_1 \sim f_2$.
\end{lem}

\begin{lem}
\label{lem-KL-sqerr}
Let $K: \mathrm{M}(I,J) \to \mathbb{R}$, $W \mapsto K(W)$ be
\begin{align}
K(W)=\int dx q'(x) \lVert (W-W_0)x \rVert^2,
\end{align}
where $W_0 \in \mathrm{M}(I,J)$.
Put $\Phi: \mathrm{M}(I,J)\to\mathbb{R}$, $W \mapsto \Phi(W)=\lVert W-W_0 \rVert^2$.
A symmetric $J \times J$ matrix whose $(i,j)$-entry is $\int x_i x_j q'(x)dx$ for $1\leqq i \leqq N$ and $1 \leqq j \leqq N$ is denoted by $\mathscr{X}$.
If $\mathscr{X}$ is positive definite, then there exist positive constants $\alpha_1,\alpha_2>0$ such that $\alpha_1 \Phi(W) \leqq K(W) \leqq \alpha_2 \Phi(W)$ holds on a neighborhood of $K^{-1}(0)$.
Hence, $K \sim \Phi$.
\end{lem}

\begin{lem}
\label{lem-KL-ce}
Let $K: \Delta_d \to \mathbb{R}$, $w \mapsto K(w)$ be
\begin{align}
K(w) = \sum_{y \in \mathrm{Onehot}(d)} \prod_{j=1}^d (w^0_j)^{y_j} \log \frac{\prod_{j=1}^d (w^0_j)^{y_j}}{\prod_{j=1}^d (w_j)^{y_j}},
\end{align}
where $w_0$ in the interior of $\Delta_d$, $\Delta_d$ is the $d$-dimensional simplex, i.e. $w=(w_j)_{j=1}^d \in \Delta_d \Rightarrow \sum_{j=1}^d w_j=1, w_j \geqq 0)$,
and $\mathrm{Onehot}(d)$ is the set of $d$-dimensional onehot vectors
\begin{align}
\mathrm{Onehot}(d)=\{ y \in \{0,1\}^d \mid y_j=1, y_l=0 \ (l \ne j), {\rm for} \ j=1,\ldots,d \}.
\end{align}
Put $\Phi: \Delta_d \to \mathbb{R}$, $w \mapsto \Phi(w)=\sum_{j=1}^{d-1} (w_j-w^0_j)^2$.
There are positive constants $\alpha_1,\alpha_2>0$ such that $\alpha_1 \Phi(w) \leqq K(w) \leqq \alpha_2 \Phi(w)$ holds on a neighborhood of $K^{-1}(0)$.
Hence, $K \sim \Phi$.
\end{lem}

%%ベルヌーイやカテゴリカル分布のKLも二乗誤差になることを補題にする
%%補題１からただちに得られる、って書き方でもいいかも。

Lemma \ref{lem-equivRLCT} was proved in \cite{SWatanabeBookE}.
Lemma \ref{lem-KL-sqerr} was proved in \cite{Aoyagi1}.
Lemma \ref{lem-KL-ce} was proved in \cite{Matsuda1-e}.
Also, by using Lemma \ref{lem-KL-ce} in the case $d=2$, the following lemma can be derived.
\begin{lem}
\label{lem-KL-bce}
Let $K: [0,1]^R \to \mathbb{R}$, $w \mapsto K(w)$ be
\begin{align}
K(w) = \sum_{c=(c_k)_{k=1}^R \in \{0,1\}^R } \prod_{l=1}^R (w^0_l)^{c_l}(1-w^0_l)^{1-c_l}
\log \frac{\prod_{k=1}^R (w^0_k)^{c_k}(1-w^0_k)^{1-c_k}}{\prod_{k=1}^R (w_k)^{c_k}(1-w_k)^{1-c_k}},
\end{align}
where $w_0$ in the interior of $[0,1]^R$.
Put $\Phi: [0,1]^R \to \mathbb{R}$, $w \mapsto \Phi(w)=\lVert w -w_0 \rVert^2$.
There are positive constants $\alpha_1,\alpha_2>0$ such that $\alpha_1 \Phi(w) \leqq K(w) \leqq \alpha_2 \Phi(w)$ holds on a neighborhood of $K^{-1}(0)$.
Hence, $K \sim \Phi$.
\end{lem}

\begin{proof}[Proof of Lemma \ref{lem-KL-bce}]
In the case of $R=1$, this lemma is equivalent to Lemma \ref{lem-KL-ce}.

In the case of $R \geqq 2$,
developing $K(w)$, we have
\begin{align}
K(w) &= \sum_{c=(c_k)_{k=1}^d \in \{0,1\}^R } \left\{ \prod_{l=1}^R (w^0_l)^{c_l}(1-w^0_l)^{1-c_l} \times \right. \\
&\quad \left. \sum_{k=1}^R \left( \log (w^0_k)^{c_k}(1-w^0_k)^{1-c_k} - \log (w_k)^{c_k}(1-w_k)^{1-c_k} \right) \right\} \\
&= \sum_{c \in \{0,1\}^R } \prod_{l=1}^R (w^0_l)^{c_l}(1-w^0_l)^{1-c_l}
\sum_{k=1}^R \log \frac{(w^0_k)^{c_k}(1-w^0_k)^{1-c_k}}{(w_k)^{c_k}(1-w_k)^{1-c_k}} \\
&= \sum_{k=1}^R \sum_{c \in \{0,1\}^R } \prod_{l=1}^R (w^0_l)^{c_l}(1-w^0_l)^{1-c_l}
\log \frac{(w^0_k)^{c_k}(1-w^0_k)^{1-c_k}}{(w_k)^{c_k}(1-w_k)^{1-c_k}}.
\end{align}
Fix an arbitrary $k \in \{1, \ldots, R \}$.
If $l \ne k$, the expectation by the $l$-th Bernoulli distribution $(w^0_l)^{c_l}(1-w^0_l)^{1-c_l}$ does not affect to the $k$-th log mass ratio:
\begin{align}
\sum_{c_l=0}^1 (w^0_l)^{c_l}(1-w^0_l)^{1-c_l} \log \frac{(w^0_k)^{c_k}(1-w^0_k)^{1-c_k}}{(w_k)^{c_k}(1-w_k)^{1-c_k}} = \log \frac{(w^0_k)^{c_k}(1-w^0_k)^{1-c_k}}{(w_k)^{c_k}(1-w_k)^{1-c_k}}.
\end{align}
This leads to the following:
\begin{align}
K(w) &= \sum_{k=1}^R \sum_{c_k=0}^1 (w^0_k)^{c_k}(1-w^0_k)^{1-c_k}
\log \frac{(w^0_k)^{c_k}(1-w^0_k)^{1-c_k}}{(w_k)^{c_k}(1-w_k)^{1-c_k}}.
\end{align}
Now, let $v(k)=(w_k, 1-w_k) \in \Delta_2$, $v^0(k)=(w^0_k, 1-w^0_k) \in \Delta_2$, and $C(k)=(c_k, 1-c_k) \in \mathrm{Onehot}(2)$.
Then, we have
\begin{align}
(w_k)^{c_k}(1-w_k)^{1-c_k} &= \prod_{j=1}^2 (v(k))_{=j}^{(C(k))_{=j}}, \\
(w^0_k)^{c_k}(1-w^0_k)^{1-c_k} &= \prod_{j=1}^2 (v^0(k))_{=j}^{(C(k))_{=j}}.
\end{align}
For simplicity, we write $(v(k))_{=j}$ and $(C(k))_{=j}$ as $v(k)_j$ and $C(k)_j$, respectively.
Since $C(k) \in \mathrm{Onehot}(2)$ and $v(k), v^0(k) \in \Delta_2$, we obtain
\begin{align}
K(w) &= \sum_{k=1}^R \sum_{c_k=0}^1 \prod_{j=1}^2 v^0(k)_j^{C(k)_j}
\log \frac{\prod_{j=1}^2 v^0(k)_j^{C(k)_j}}{\prod_{j=1}^2 v(k)_j^{C(k)_j}} \\
&= \sum_{k=1}^R \sum_{C(k) \in \mathrm{Onehot}(2)} \prod_{j=1}^2 v^0(k)_j^{C(k)_j}
\log \frac{\prod_{j=1}^2 v^0(k)_j^{C(k)_j}}{\prod_{j=1}^2 v(k)_j^{C(k)_j}}.
\end{align}
Put
\begin{align}
\psi(v(k)) = \sum_{C(k) \in \mathrm{Onehot}(2)} \prod_{j=1}^2 v^0(k)_j^{C(k)_j}
\log \frac{\prod_{j=1}^2 v^0(k)_j^{C(k)_j}}{\prod_{j=1}^2 v(k)_j^{C(k)_j}}.
\end{align}
Applying Lemma \ref{lem-KL-ce},
there exist $R$-dimensional vectors $\alpha^1 = (\alpha^1_k)_{k=1}^R$ and $\alpha^2 = (\alpha^2_k)_{k=1}^R$ whose entries are positive constants such that
\begin{align}
\alpha^1_k (w_k-w^0_k)^2 \leqq \psi(v(k)) \leqq \alpha^2_k (w_k-w^0_k)^2,
\end{align}
on a neighborhood of $\psi^{-1}(0)$ for $k=1, \ldots, R$.
Summarizing them, we get
\begin{align}
\sum_{k=1}^R \alpha^1_k (w_k-w^0_k)^2 \leqq \sum_{k=1}^R \psi(v(k)) \leqq \sum_{k=1}^R \alpha^2_k (w_k-w^0_k)^2.
\end{align}
Because of $\Phi(w) =\sum_{k=1}^R (w_k-w^0_k)^2$, we have
\begin{align}
\min_k \{\alpha^1_k\} \Phi(w) \leqq K(w) \leqq \max_k \{ \alpha^2_k \} \Phi(w).
\end{align}
Therefore, $K \sim \Phi$.
\end{proof}

The above lemmas indicate that equivalent discrepancies have same RLCTs and examples of such discrepancies are the KL divergences between Gaussian, categorical, and Bernoulli distributions.

Here, we prove Theorem \ref{thm-main-cbm}.
%20230315ここまで

\begin{proof}[Proof of Theorem \ref{thm-main-cbm}]
By using
\begin{align}
\lVert y-ABx \rVert^2 &= \langle y-ABx, y-ABx \rangle \\
&= \lVert y \rVert^2 -2\langle y, ABx \rangle + \lVert ABx \rVert^2
\end{align}
and that of $y-A_0B_0x$, $c-Bx$, and $c-B_0x$,
we expand $\log q_1/p_1$ as follows:

\begin{align}
\log \frac{q_1(y,c|x)}{p_1(y,c|A,B,x)}
&= \log \frac{\exp\left(-\frac{1}{2}\lVert y-A_0B_0x \rVert^2\right)\exp\left(-\frac{\gamma}{2}\lVert c-B_0x \rVert^2\right)}{\exp\left(-\frac{1}{2}\lVert y-ABx \rVert^2\right)\exp\left(-\frac{\gamma}{2}\lVert c-Bx \rVert^2\right)} \\
&= -\frac{1}{2}(\lVert y \rVert^2 -2\langle y, A_0B_0x \rangle + \lVert A_0B_0x \rVert^2)
-\frac{\gamma}{2}(\lVert c \rVert^2 -2\langle c, B_0x \rangle + \lVert B_0x \rVert^2) \\
&\quad +\frac{1}{2}(\lVert y \rVert^2 -2\langle y, ABx \rangle + \lVert ABx \rVert^2)
+\frac{\gamma}{2}(\lVert c \rVert^2 -2\langle c, Bx \rangle + \lVert Bx \rVert^2) \\
&=\frac{1}{2}(\lVert ABx \rVert^2 -2\langle y, (AB-A_0B_0)x \rangle - \lVert A_0B_0x \rVert^2) \\
&\quad +\frac{\gamma}{2}(\lVert Bx \rVert^2 -2\langle c, (B-B_0)x \rangle - \lVert B_0x \rVert^2).
\end{align}
Averaging by $q_{1}(y,c|x)$, we have
\begin{align}
\iint dcdy q_1(y,c|x)\log \frac{q_1(y,c|x)}{p_1(y,c|A,B,x)}
&= \frac{1}{2}(\lVert ABx \rVert^2 -2\langle A_0B_0x, (AB-A_0B_0)x \rangle - \lVert A_0B_0x \rVert^2) \\
&\quad +\frac{\gamma}{2}(\lVert Bx \rVert^2 -2\langle B_0x, (B-B_0)x \rangle - \lVert B_0x \rVert^2) \\
&= \frac{1}{2}(\lVert ABx \rVert^2 -2\langle A_0B_0x, ABx \rangle + \lVert A_0B_0x \rVert^2) \\
&\quad +\frac{\gamma}{2}(\lVert Bx \rVert^2 -2\langle B_0x, Bx \rangle + \lVert B_0x \rVert^2) \\
&= \frac{1}{2}( \lVert (AB-A_0B_0)x \rVert^2 + \gamma \lVert (B-B_0)x \rVert^2).
\end{align}
Let $\Psi_1(A,B)=(1/2)\int dx q'(x) \lVert (AB-A_0B_0)x \rVert^2$ and $\Psi_2(B)=(1/2)\int dx q'(x) \lVert (B-B_0)x \rVert^2$.
Because of Lemma \ref{lem-KL-sqerr}, there are positive constants $c_1, c_2, c_3, c_4 >0$ such that
\begin{align}
\label{pf-thm1-ABineq}
c_1 \lVert AB-A_0B_0 \rVert^2 \leqq \Psi_1(A,B) \leqq c_2 \lVert AB-A_0B_0 \rVert^2, \\
\label{pf-thm1-Bineq}
c_3 \lVert B-B_0 \rVert^2 \leqq \gamma \Psi_2(B) \leqq c_4 \lVert B-B_0 \rVert^2.
\end{align}
Thus, by adding Eq. (\ref{pf-thm1-ABineq}) to  Eq. (\ref{pf-thm1-Bineq}),
we have
\begin{align}
c_1 \lVert AB-A_0B_0 \rVert^2 + c_3 \lVert B-B_0 \rVert^2
\leqq \Psi_1(A,B) + \gamma \Psi_2(B)
\leqq c_2 \lVert AB-A_0B_0 \rVert^2 + c_4 \lVert B-B_0 \rVert^2.
\end{align}
Let $C_1$ and $C_2$ be $\min\{c_1,c_3\}$ and $\max\{c_2, c_4 \}$, respectively.
We immediately obtain
\begin{align}
C_1 (\lVert AB-A_0B_0 \rVert^2 + \lVert B-B_0 \rVert^2)
\leqq \Psi_1(A,B) + \gamma \Psi_2(B)
\leqq C_2 (\lVert AB-A_0B_0 \rVert^2 + \lVert B-B_0 \rVert^2).
\end{align}
Applying Lemma \ref{lem-equivRLCT} to the above inequality, we have
\begin{align}
\Psi_1(A,B) + \gamma \Psi_2(B) \sim \lVert AB-A_0B_0 \rVert^2 + \lVert B-B_0 \rVert^2.
\end{align}
Therefore,
\begin{align}
K_1(A,B) &=\iiint dxdcdy q'(x) q_1(y,c|x)\log \frac{q_1(y,c|x)}{p_1(y,c|A,B,x)} \\
&=\Psi_1(A,B) + \gamma \Psi_2(B) \\
&\sim \lVert AB-A_0B_0 \rVert^2 + \lVert B-B_0 \rVert^2.
\end{align}

To determine the RLCT $\lambda_1$ and its multiplicity $m_1$, we should consider the following analytic set
\begin{align}
\mathcal{V}_1 = \{(A,B) \mid \lVert AB-A_0B_0 \rVert^2+ \lVert B-B_0 \rVert^2=0, A \in \mathrm{M}(M,K) \ {\rm and } \ B \in \mathrm{M}(K,N) \}.
\end{align}
We take $\lVert AB-A_0B_0 \rVert^2+ \lVert B-B_0 \rVert^2=0$.
Because $\lVert AB-A_0B_0 \rVert^2 \geqq 0$ and $\lVert B-B_0 \rVert^2 \geqq 0$ hold,
we have
\begin{align}
\lVert AB-A_0B_0 \rVert^2+ \lVert B-B_0 \rVert^2=0 \Leftrightarrow \lVert AB-A_0B_0 \rVert^2=0 \ {\rm and} \ \lVert B-B_0 \rVert^2=0.
\end{align}
Hence, $AB=A_0B_0$ and $B=B_0$, i.e. $(A,B)=(A_0,B_0)$.
Therefore, $\mathcal{V}_1=\{(A_0,B_0) \}$.
This means that there is no singularity, i.e. the model is regular.
Hence, the RLCT is equal to a half of the parameter dimension \cite{SWatanabeBookE}.
Since the parameter dimension equals $(M+N)K$, we have
\begin{align}
\lambda_1 = \frac{1}{2}(M+N)K, \ m_1=1.
\end{align}
\end{proof}

Next, we prove Theorem \ref{thm-main-multitask}.
This is immediately derived using the Aoyagi's theorem \cite{Aoyagi1} as follows.
\begin{thm}[Aoyagi and Watanabe]
\label{thm-aoyagi-rrr}
%Suppose that the $N \times N$ matrix $((\int x_i x_j q'(x) dx)_{ij})_{i=1,j=1}^{N,N}$ is positive definite.
Suppose $\mathscr{X}$ is positive definite.
Let $\lambda_3$ be the RLCT of Standard of three-layered linear neural network and $m_3$ be its multiplicity.
By using the input dimension $N$, the number of intermediate units $H$, the output dimension $M$, and the true rank $H_0$,
they can be represented as follows:
\begin{enumerate}
\item In the case $M+H_0 \leqq N+H$ and $N+H_0 \leqq M+H$ and $H+H_0 \leqq N+M$,
    \begin{enumerate}
        \item and if $N+M+H+H_0$ is even, then
        $$\lambda_3 = \frac{1}{8}\{ 2(H+H_0)(N+M)-(N-M)^2-(H+H_0)^2 \}, \ m_3 = 1.$$
        \item and if $N+M+K+H+H_0$ is odd, then
        $$\lambda_3 = \frac{1}{8}\{ 2(H+H_0)(N+M)-(N-M)^2-(H+H_0)^2 +1 \}, \ m_3 = 2.$$
    \end{enumerate}
\item In the case $N+H<M+H_0$, then
$$\lambda_3 = \frac{1}{2}\{HN+H_0(M-H)\}, \ m_3 = 1.$$
\item In the case $M+H<N+H_0$, then
$$\lambda_3 = \frac{1}{2}\{HM+H_0(N-H)\}, \ m_3 = 1.$$
\item Otherwise (i.e. $N+M<H+H_0$), then
$$\lambda_3 = \frac{1}{2}NM, \ m_3 = 1.$$
\end{enumerate}
\end{thm}

\begin{proof}[Proof of Theoem \ref{thm-main-multitask}]
In the case of Multitask, the output dimension is expanded from $M$ to $M+K$ since Multitask makes the output and the concept co-occur to derive the explanation.
Mathematically, this involves simply increasing the dimension.
Therefore, by plug-inning $M+K$ to $M$ in Theorem \ref{thm-aoyagi-rrr}, we obtain Theorem \ref{thm-main-multitask}.
\end{proof}

We expand the main results to the case mentioned in section \ref{sec:Expansion}.

\begin{proof}[Proof of Theorem \ref{thm-expand-cbm}]
Put $u=(u_j)_{j=1}^M=s_{M}(AB x)$, $u_0=(u^0_j)_{j=1}^M=s_{M}(A_0B_0 x)$, $v=(v_k)_{k=1}^K= \sigma_{K}(B x)$, and $v_0=(v^0_k)_{k=1}^K=\sigma_{K}(B_0 x)$.

(1) In the case when $i=1$ and $j=1$, this is Theorem \ref{thm-main-cbm}.

(2) In the case when $i=1$ and $j=2$, we expand $\log q_1^{12}/ p_1^{12}$ as the following:

\begin{align}
\log \frac{q_1^{12}(y,c|x)}{p_1^{12}(y,c|A,B,x)}
&= \log \frac{\exp\left(-\frac{1}{2}\lVert y-A_0B_0x \rVert^2\right)
\left(\prod_{k=1}^K (\sigma_{K}(B_0 x))_k^{c_k} (1-(\sigma_{K}(B_0 x))_k)^{1-c_k} \right)^{\gamma}}
{\exp\left(-\frac{1}{2}\lVert y-ABx \rVert^2\right)
\left( \prod_{k=1}^K (\sigma_{K}(B x))_k^{c_k} (1-(\sigma_{K}(B x))_k)^{1-c_k} \right)^{\gamma}} \\
&= -\frac{1}{2}(\lVert y \rVert^2 -2\langle y, A_0B_0x \rangle + \lVert A_0B_0x \rVert^2)
+\gamma \log \prod_{k=1}^K (v^0_k)^{c_k}(1-v^0_k)^{1-c_k} \\
&\quad +\frac{1}{2}(\lVert y \rVert^2 -2\langle y, ABx \rangle + \lVert ABx \rVert^2)
-\gamma \log \prod_{k=1}^K (v_k)^{c_k}(1-v_k)^{1-c_k} \\
&=\frac{1}{2}(\lVert ABx \rVert^2 -2\langle y, (AB-A_0B_0)x \rangle - \lVert A_0B_0x \rVert^2) \\
&\quad + \gamma\log \frac{\prod_{k=1}^K (v^0_k)^{c_k}(1-v^0_k)^{1-c_k}}{\prod_{k=1}^K (v_k)^{c_k}(1-v_k)^{1-c_k}}
%&\quad +\frac{\gamma}{2}(\lVert Bx \rVert^2 -2\langle c, (B-B_0)x \rangle - \lVert B_0 \rVert^2).
\end{align}
Integrating by $dydc q_1^{12}(y,c|x)$, we have
\begin{align}
&\quad \iint dydc q_1^{12}(y,c|x) \log \frac{q_1^{12}(y,c|x)}{p_1^{12}(y,c|A,B,x)} \\
&= \frac{1}{2}\lVert (AB-A_0B_0)x \rVert^2 +
\gamma \sum_{c=(c_k)_{k=1}^K \in \{0,1\}^K } \prod_{l=1}^K (v^0_l)^{c_l}(1-v^0_l)^{1-c_l} \log \frac{\prod_{k=1}^K (v^0_k)^{c_k}(1-v^0_k)^{1-c_k}}{\prod_{k=1}^K (v_k)^{c_k}(1-v_k)^{1-c_k}}.
\end{align}
According to Lemma \ref{lem-KL-bce}, the second term $\psi(v)$ is evaluated by $\lVert v-v_0 \rVert^2$ as shown below; there are positive constants $\alpha_1, \alpha_2>0$ such that
\begin{align}
\alpha_1 \lVert v-v_0 \rVert^2 \leqq
\psi(v)
\leqq \alpha_2 \lVert v-v_0 \rVert^2,
\end{align}
where
\begin{align}
\psi(v) = \gamma \sum_{c=(c_k)_{k=1}^K \in \{0,1\}^K } \prod_{k=1}^K (v^0_k)^{c_k}(1-v^0_k)^{1-c_k} \log \frac{\prod_{k=1}^K (v^0_k)^{c_k}(1-v^0_k)^{1-c_k}}{\prod_{k=1}^K (v_k)^{c_k}(1-v_k)^{1-c_k}}.
\end{align}
Let $\beta_1 = \min \{1, \alpha_1 \}$ and $\beta_2 = \max \{ 1, \alpha_2 \}$.
Because $\lVert (AB-A_0B_0)x \rVert^2/2 \geqq 0$ holds, adding it to the both sides, we have
\begin{align}
\beta_1 \left(\frac{1}{2}\lVert (AB-A_0B_0)x \rVert^2 + \lVert v-v_0 \rVert^2 \right)
&\leqq \frac{1}{2}\lVert (AB-A_0B_0)x \rVert^2 + \psi(v) \\
&\leqq \beta_2 \left(\frac{1}{2}\lVert (AB-A_0B_0)x \rVert^2 + \lVert v-v_0 \rVert^2 \right),
\end{align}
Thus, we should consider $\lVert (AB-A_0B_0)x \rVert^2/2 + \lVert v-v_0 \rVert^2$.
With $v=\sigma_K(Bx)$ and $v_0=\sigma_K(B_0x)$, we have
\begin{align}
\frac{1}{2}\lVert (AB-A_0B_0)x \rVert^2 + \lVert v-v_0 \rVert^2
= \frac{1}{2}\lVert (AB-A_0B_0)x \rVert^2 + \lVert \sigma_K(Bx)-\sigma_K(B_0x) \rVert^2.
\end{align}
On account of Lemma \ref{lem-KL-sqerr}, the average of the first term by $dxq'(x)$ is equivalent to $\lVert AB-A_0B_0 \rVert^2$.
On the other hand, since $\sigma_K$ is analytic isomorphic onto its image and does not have parameters, the averaged second term has the same RLCT of linear regression $Bx-B_0x$ (both models are regular), i.e.
\begin{align}
\int dxq'(x)\lVert \sigma_K(Bx)-\sigma_K(B_0x) \rVert^2
&\sim \int dxq'(x)\lVert (B-B_0)x \rVert^2 \\
&\sim \lVert B-B_0 \rVert^2.
\end{align}
Hence,
\begin{align}
\iiint dxdydc q'(x)q_1^{12}(y,c|x) \log \frac{q_1^{12}(y,c|x)}{p_1^{12}(y,c|A,B,x)}
\sim \lVert AB-A_0B_0 \rVert^2 + \lVert B-B_0 \rVert^2
\end{align}
holds and results in Theorem \ref{thm-main-cbm}.

(3) In the case when $i=2$ and $j=1$,
similar to the case of $(i,j)=(1,2)$,
because of
\begin{align}
\log \frac{q_1^{21}(y,c|x)}{p_1^{21}(y,c|A,B,x)}
&= \log \frac{\prod_{j=1}^{M} (s_{M}(A_0B_0x))_j^{y_j} \exp\left(-\frac{\gamma}{2}\lVert c-B_0x \rVert^2\right)}{\prod_{j=1}^{M} (s_{M}(ABx))_j^{y_j} \exp\left(-\frac{\gamma}{2}\lVert c-Bx \rVert^2\right)} \\
&= \log\prod_{j=1}^{M} (s_{M}(A_0B_0x))_j^{y_j}
-\frac{\gamma}{2}(\lVert c \rVert^2 -2\langle c, B_0x \rangle + \lVert B_0x \rVert^2) \\
&\quad -\log\prod_{j=1}^{M} (s_{M}(ABx))_j^{y_j}
+\frac{\gamma}{2}(\lVert c \rVert^2 -2\langle c, Bx \rangle + \lVert B_0x \rVert^2) \\
&=\log\frac{\prod_{j=1}^{M} (s_{M}(A_0B_0x))_j^{y_j}}{\prod_{j=1}^{M} (s_{M}(ABx))_j^{y_j}} +\frac{\gamma}{2}(\lVert Bx \rVert^2 -2\langle c, (B-B_0)x \rangle - \lVert B_0x \rVert^2),
\end{align}
we have
\begin{align}
&\quad \iint dydc q_1^{12}(y,c|x) \log \frac{q_1^{12}(y,c|x)}{p_1^{12}(y,c|A,B,x)} \\
&= \sum_{y \in \mathrm{Onehot}(M)}\prod_{j=1}^{M} (u^0_j)^{y_j}\log\frac{\prod_{j=1}^{M} (u^0_j)^{y_j}}{\prod_{j=1}^{M} (u_j)^{y_j}}
+\frac{\gamma}{2}(\lVert (B-B_0)x \rVert^2,
\end{align}
by using $u=s_M(ABx)$ and $u_0=s_M(A_0B_0x)$.
Owing to Lemma \ref{lem-KL-ce} and \ref{lem-KL-sqerr},
the first and second terms averaged by $dxq'(x)$ have same RLCTs of the average of $\sum_{j=1}^{M-1}(u_j - u^0_j)^2$ and $\lVert B-B_0 \rVert^2$, respectively.
Since a map $(w)_{<M} \mapsto (s_{M}(w))_{<M}$ is analytic and isomorphic onto its image,
we obtain
\begin{align}
\int dxq'(x)\sum_{j=1}^{M-1}(u_j - u^0_j)^2
&=\int dxq'(x) \lVert (u)_{<M} - (u_0)_{<M} \rVert^2 \\
&\sim \int dxq'(x)\lVert (ABx)_{<M}-(A_0B_0x)_{<M} \rVert^2 \\
&= \int dxq'(x)\lVert ((A)_{<M}B-(A_0)_{<M}B_0)x \rVert^2 \\
&\sim \lVert (A)_{<M}B-(A_0)_{<M}B_0 \rVert^2.
\end{align}
Therefore, we have
\begin{align}
\iiint dxdydc q'(x)q_1^{12}(y,c|x) \log \frac{q_1^{12}(y,c|x)}{p_1^{12}(y,c|A,B,x)}
\sim \lVert  (A)_{<M}B-(A_0)_{<M}B_0 \rVert^2 + \lVert B-B_0 \rVert^2.
\end{align}
Similar to the proof of Theorem \ref{thm-main-cbm},
the zero point set of the above function is $((A)_{<M},B)=((A_0)_{<M},B_0)$.
This leads to the following:
\begin{align}
\lambda_1^{12}=\frac{1}{2}(M+N-1)K, \\
m_1^{12} = 1.
\end{align}

(4) In the case when $i=2$ and $j=2$, the KL divergence can be developed in the same way as that in the case of $(i,j)=(1,2)$ and $(2,1)$.
Therefore, we have
\begin{align}
K_1^{22}(A,B) &= \int dx q'(x) \left( \sum_{y \in \mathrm{Onehot}(M)}\prod_{j=1}^{M} (u^0_j)^{y_j}\log\frac{\prod_{j=1}^{M} (u^0_j)^{y_j}}{\prod_{j=1}^{M} (u_j)^{y_j}} \right. \\
&\quad \left. + \gamma \sum_{c=(c_k)_{k=1}^K \in \{0,1\}^K } \prod_{k=1}^K (v^0_k)^{c_k}(1-v^0_k)^{1-c_k} \log \frac{\prod_{k=1}^K (v^0_k)^{c_k}(1-v^0_k)^{1-c_k}}{\prod_{k=1}^K (v_k)^{c_k}(1-v_k)^{1-c_k}} \right), \\
&\quad u=s_M(ABx), \ u_0=s_M(A_0B_0x), \ v=\sigma_K(Bx), \ v_0=\sigma_K(B_0x).
\end{align}
Similar to
\begin{align}
\sum_{y \in \mathrm{Onehot}(M)}\prod_{j=1}^{M} (u^0_j)^{y_j}\log\frac{\prod_{j=1}^{M} (u^0_j)^{y_j}}{\prod_{j=1}^{M} (u_j)^{y_j}}
\sim \lVert (A)_{<M}B-(A_0)_{<M}B_0 \rVert^2
\end{align}
when $i=2$ and $j=1$ and
\begin{align}
\gamma \sum_{c=(c_k)_{k=1}^K \in \{0,1\}^K } \prod_{k=1}^K (v^0_k)^{c_k}(1-v^0_k)^{1-c_k} \log \frac{\prod_{k=1}^K (v^0_k)^{c_k}(1-v^0_k)^{1-c_k}}{\prod_{k=1}^K (v_k)^{c_k}(1-v_k)^{1-c_k}}
\sim \lVert B-B_0 \rVert^2
\end{align}
when $i=1$ and $j=2$, we obtain
\begin{align}
K_1^{22}(A,B) \sim \lVert (A)_{<M}B-(A_0)_{<M}B_0 \rVert^2 + \lVert B-B_0 \rVert^2.
\end{align}
This is the same when $i=2$ and $j=1$ and
\begin{align}
\lambda_1^{22}=\frac{1}{2}(M+N-1)K, \\
m_1^{22} = 1.
\end{align}

Based on (1), (2), (3), and (4) noted above, the theorem is therefore proved.
\end{proof}

\begin{proof}[Proof of Theorem \ref{thm-expand-multitask}]
Put $u=(u_j)_{j=1}^M=s_{M}((UV x)^\mathrm{y})$, $u_0=(u^0_j)_{j=1}^M=s_{M}((U_0V_0 x)^\mathrm{y})$, $v=(v_k)_{k=1}^K=\sigma_{K}((UVx)^\mathrm{c})$ and $v_0=(v^0_k)_{k=1}^K=\sigma_{K}((U_0V_0x)^\mathrm{c})$.
Also, set $\overline{U}=(U)_{\leqq M}$, $\overline{U_0}=(U_0)_{\leqq M}$, $\underline{U}=(U)_{>M}$, and $\underline{U}=(U_0)_{> M}$.

(1) In the case when $i=1$ and $j=1$, this is same as Theorem \ref{thm-main-multitask}.

(2) In the case when $i=1$ and $j=2$.
Using the same expansion method as that of $K_1^{12}(A,B)$ in the proof of Theorem \ref{thm-expand-cbm}, we have
\begin{align}
\log \frac{q_2^{12}(y,c|x)}{p_2^{12}(y,c|U,V,x)}
&=\frac{1}{2}(\lVert (UVx)^\mathrm{y} \rVert^2 -2\langle y, (UVx)^\mathrm{y}-(U_0V_0x)^\mathrm{y} \rangle - \lVert (U_0V_0x)^\mathrm{y} \rVert^2) \\
&\quad + \gamma\log \frac{\prod_{k=1}^H (v^0_k)^{c_k}(1-v^0_k)^{1-c_k}}{\prod_{k=1}^H (v_k)^{c_k}(1-v_k)^{1-c_k}} \\
&=\frac{1}{2}(\lVert (UVx)_{\leqq M} \rVert^2 -2\langle y, (UVx)_{\leqq M}-(U_0V_0x)_{\leqq M} \rangle - \lVert (U_0V_0x)_{\leqq M} \rVert^2) \\
&\quad + \gamma\log \frac{\prod_{k=1}^H (v^0_k)^{c_k}(1-v^0_k)^{1-c_k}}{\prod_{k=1}^H (v_k)^{c_k}(1-v_k)^{1-c_k}} \\
&=\frac{1}{2}(\lVert \overline{U}Vx \rVert^2 -2\langle y, (\overline{U}V-\overline{U_0}V_0)x \rangle - \lVert \overline{U_0}V_0x \rVert^2) \\
&\quad + \gamma\log \frac{\prod_{k=1}^H (v^0_k)^{c_k}(1-v^0_k)^{1-c_k}}{\prod_{k=1}^H (v_k)^{c_k}(1-v_k)^{1-c_k}}.
\end{align}
Equally averaging similar to the proof of Theorem\ref{thm-expand-cbm} when $(i,j)=(1,2)$ and applying Lemma \ref{lem-KL-bce},
we obtain
\begin{align}
K_2^{12}(U,V) &= \int dx q'(x) \left(
\lVert (\overline{U}V-\overline{U_0}V_0)x \rVert^2 \right. \\
&\quad \left. + \gamma \sum_{c=(c_k)_{k=1}^d \in \{0,1\}^H } \prod_{k=1}^H (v^0_k)^{c_k}(1-v^0_k)^{1-c_k} \log \frac{\prod_{k=1}^H (v^0_k)^{c_k}(1-v^0_k)^{1-c_k}}{\prod_{k=1}^H (v_k)^{c_k}(1-v_k)^{1-c_k}}
\right) \\
&\sim \int dx q'(x) \left( \lVert (\overline{U}V-\overline{U_0}V_0)x \rVert^2 + \lVert \sigma_{K}(\underline{U}Vx) - \sigma_{K}(\underline{U_0}V_0x) \rVert^2 \right).
\end{align}
$\sigma_K$ is analytic isomorphic onto its image and that leads the following:
\begin{align}
K_2^{12}(U,V) &\sim \int dx q'(x) \left( \lVert (\overline{U}V-\overline{U_0}V_0)x \rVert^2 + \lVert \sigma_{K}(\underline{U}Vx) - \sigma_{K}(\underline{U_0}V_0x) \rVert^2 \right) \\
&\sim  \lVert \overline{U}V-\overline{U_0}V_0 \rVert^2 + \lVert \underline{U}V-\underline{U_0}V_0 \rVert^2 \\
&= \lVert UV-U_0V_0 \rVert^2.
\end{align}
Thus, all we have to compute is the RLCT of $\lVert UV-U_0V_0 \rVert^2$.
This is the same as Theorem \ref{thm-main-multitask}.

(3) In the case when $i=2$ and $j=1$.
Since $(w)_{<M} \to s_M(w)$, $w \in \mathbb{R}^{M-1}$ is analytic isomorphic onto its image,
using the same method as in the case of $i=1$ and $j=2$, we have
\begin{align}
K_2^{21}(U,V)
&\sim \int dx q'(x) \left( \lVert s_M(\overline{U}Vx)-s_M(\overline{U_0}V_0x) \rVert^2 + \lVert (\underline{U}V - \underline{U_0}V_0)x \rVert^2 \right) \\
&\sim \lVert (U)_{<M}V-(U_0)_{<M}V_0 \rVert^2 + \lVert \underline{U}V-\underline{U_0}V_0 \rVert^2 \\
&= \lVert (U)_{<M}V-(U_0)_{<M}V_0 \rVert^2 + \lVert (U)_{>M}V-(U_0)_{>M}V_0 \rVert^2 \\
&= \lVert (U)_{\ne M}V-(U_0)_{\ne M}V_0 \rVert^2.
\end{align}
$(U)_{\ne M} \in \mathrm{M}(M+K-1,H)$ and $(U_0)_{\ne M} \in \mathrm{M}(M+K-1,H_0)$ hold.
Therefore, the RLCT of $\lVert (U)_{\ne M}V-(U_0)_{\ne M}V_0 \rVert^2$ is calculated by assigning $M-1$ to $M$ in Theorem \ref{thm-main-multitask}.
This leads to the following:
\begin{align}
\lambda_2^{21} &= \lambda_2(N,H,M-1,K,H_0), \\
m_2^{21} &= m_2(N,H,M-1,K,H_0).
\end{align}

(4) In the case when $i=2$ and $j=2$, combining the above methods, we have
\begin{align}
K_2^{22}(U,V) &\sim \int dx q'(x) \left( \lVert s_M(\overline{U}Vx)-s_M(\overline{U_0}V_0x) \rVert^2 + \lVert \sigma_{K}(\underline{U}Vx) - \sigma_{K}(\underline{U_0}V_0x) \rVert^2 \right) \\
&\sim \lVert (U)_{<M}V-(U_0)_{<M}V_0 \rVert^2 + \lVert \underline{U}V-\underline{U_0}V_0 \rVert^2 \\
&= \lVert (U)_{\ne M}V-(U_0)_{\ne M}V_0 \rVert^2.
\end{align}
This results in the case when $i=2$ and $j=1$.
Therefore, we obtain
\begin{align}
\lambda_2^{22} &= \lambda_2(N,H,M-1,K,H_0), \\
m_2^{22} &= m_2(N,H,M-1,K,H_0).
\end{align}

Therefore, based on (1), (2), (3), and (4), the theorem is proved.
\end{proof}

Lastly, we prove Proposition \ref{prop-compare}.

\begin{proof}[Proof of Proposition \ref{prop-compare}]
Develop $\lambda_1 - \lambda_2$ and solve $\lambda_1>\lambda_2$ for each cases.
If they are resolved, the opposite case $\lambda_1 \leqq \lambda_2$ can immediately be derived. 

(1) In the case of $M+K+H_0 \leqq N+H$ and $N+H_0 \leqq M+K+H$ and $H+H_0 \leqq N+M+K$. When $N+M+K+H+H_0$ is even, we have
\begin{align}
\lambda_1 - \lambda_2 &= \frac{1}{8}\{4MK+4NK-2(H+H_0)(N+M+K)+(N-M-K)^2+(H+H_0)^2\} \\
&=\frac{1}{8}\{ K^2 +2(M+N-H-H_0)K +N^2+M^2+H^2+H_0^2 \\
&\quad -2HN-2H_0N-2MN-2HM-2H_0M+2HH_0\} \\
&=\frac{1}{8}\{ K^2 +2(M+N-H-H_0)K + (M+N-H-H_0)^2-4MN \} \\
&=\frac{1}{8}\{ (K+M+N-H-H_0)^2 -4MN \}.
\end{align}
The most right hand side of the above equation is a quadratic function of $K$ and its coefficient of the greatest order term is positive.
From this assumption, $H+H_0 \leqq M+N+K$, i.e. $K \geqq H+H_0-M-N$ holds.
Thus, solving $\lambda_1 > \lambda_2$, we obtain
\begin{align}
K &> H+H_0-M-N+2\sqrt{MN} \\
&= H+H_0-(\sqrt{M}-\sqrt{N})^2.
\end{align}
The converse can also be verified as follows:
\begin{align}
\lambda_1 > \lambda_2 \Leftrightarrow K>H+H_0-M-N+(\sqrt{M}-\sqrt{N})^2.
\end{align}
When $N+M+K+H+H_0$ is even, by following the same procedure as shown above, we have
\begin{align}
\lambda_1 - \lambda_2 &=\frac{1}{8}\{ (K+M+N-H-H_0)^2 -4MN-1 \}.
\end{align}
and
\begin{align}
\lambda_1 > \lambda_2 \Leftrightarrow K>H+H_0-M-N+\sqrt{4MN+1}.
\end{align}

(2) In the case of $N+H<M+K+H_0$, we have
\begin{align}
\lambda_1-\lambda_2 &= \frac{1}{2}\{MK+NK-HN-H_0(M+K-H)\} \\
&= \frac{1}{2}\{(M+N-H_0)K-H(N-H_0)-MH_0\}.
\end{align}
Hence, we obtain
\begin{align}
\lambda_1 > \lambda_2 \Leftrightarrow (M+N-H_0)K > H(N-H_0)+MH_0.
\end{align}

(3) In the case of $M+K+H<N+H_0$, we have
\begin{align}
\lambda_1-\lambda_2 &= \frac{1}{2}\{MK+NK-H(M+K)-H_0(N-H)\} \\
&= \frac{1}{2}\{(M+N-H)K-H_0(N-H)-MH\}.
\end{align}
Hence, we obtain
\begin{align}
\lambda_1 > \lambda_2 \Leftrightarrow (M+N-H)K > H_0(N-H)+MH.
\end{align}

(4) In the case of $N+M+K<H+H_0$, we have
\begin{align}
\lambda_1-\lambda_2 &= \frac{1}{2}(MK+NK-NM-NK) \\
&= \frac{1}{2}M(K-N).
\end{align}
Hence, we obtain
\begin{align}
\lambda_1 > \lambda_2 \Leftrightarrow K>N.
\end{align}

Therefore, based on (1), (2), (3), and (4), it can be said that this theorem holds.
\end{proof}

\bibliography{library}

\begin{thebibliography}{10}

\bibitem{AkaikeAICconf}
Hirotogu Akaike.
\newblock Information theory and an extension of the maximum likelihood
  principle.
\newblock In {\em Proceedings of the 2nd International Symposium on Information
  Theory}, volume~1, pages 267--281, 1973.

\bibitem{AkaikeAIC}
Hirotogu Akaike.
\newblock Information theory and an extension of the maximum likelihood
  principle.
\newblock In {\em Selected papers of hirotugu akaike}, pages 199--213.
  Springer, 1998.

\bibitem{Akaike1974AIC}
Hirotugu Akaike.
\newblock A new look at the statistical model identification.
\newblock {\em IEEE transactions on automatic control}, 19(6):716--723, 1974.

\bibitem{Aoyagi2}
Miki Aoyagi.
\newblock Stochastic complexity and generalization error of a restricted
  boltzmann machine in bayesian estimation.
\newblock {\em Journal of Machine Learning Research}, 11(Apr):1243--1272, 2010.

\bibitem{Aoyagi3}
Miki Aoyagi.
\newblock Learning coefficient in bayesian estimation of restricted boltzmann
  machine.
\newblock {\em Journal of Algebraic Statistics}, 4(1):30--57, 2013.

\bibitem{Aoyagi1}
Miki Aoyagi and Sumio Watanabe.
\newblock Stochastic complexities of reduced rank regression in bayesian
  estimation.
\newblock {\em Neural Networks}, 18(7):924--933, 2005.

\bibitem{Atiyah1970resolution}
Michael~Francis Atiyah.
\newblock Resolution of singularities and division of distributions.
\newblock {\em Communications on pure and applied mathematics}, 23(2):145--150,
  1970.

\bibitem{Bernstein1972}
Joseph Bernstein.
\newblock The analytic continuation of generalized functions with respect to a
  parameter.
\newblock {\em Funktsional'nyi Analiz i ego Prilozheniya}, 6(4):26--40, 1972.

\bibitem{Cemgil}
Ali~T. Cemgil.
\newblock Bayesian inference in non-negative matrix factorisation models.
\newblock {\em Computational Intelligence and Neuroscience}, 2009(4):17, 2009.
\newblock Article ID 785152.

\bibitem{Chen2021EthicalMLSurveyHealthcare}
Irene~Y. Chen, Emma Pierson, Sherri Rose, Shalmali Joshi, Kadija Ferryman, and
  Marzyeh Ghassemi.
\newblock Ethical machine learning in healthcare.
\newblock {\em Annual Review of Biomedical Data Science}, 4(1):123--144, 2021.

\bibitem{Donahue2014decaf}
Jeff Donahue, Yangqing Jia, Oriol Vinyals, Judy Hoffman, Ning Zhang, Eric
  Tzeng, and Trevor Darrell.
\newblock Decaf: A deep convolutional activation feature for generic visual
  recognition.
\newblock In {\em International conference on machine learning}, pages
  647--655. PMLR, 2014.

\bibitem{Dong2021DLSurvey}
Shi Dong, Ping Wang, and Khushnood Abbas.
\newblock A survey on deep learning and its applications.
\newblock {\em Computer Science Review}, 40:100379, 2021.

\bibitem{Drton2009LCARLCT}
Mathias Drton.
\newblock Likelihood ratio tests and singularities.
\newblock {\em The Annals of Statistics}, 37(2):979--1012, 2009.

\bibitem{Drton2017forest}
Mathias Drton, Shaowei Lin, Luca Weihs, Piotr Zwiernik, et~al.
\newblock Marginal likelihood and model selection for gaussian latent tree and
  forest models.
\newblock {\em Bernoulli}, 23(2):1202--1232, 2017.

\bibitem{Drton}
Mathias Drton and Martyn Plummer.
\newblock A bayesian information criterion for singular models.
\newblock {\em Journal of the Royal Statistical Society Series B}, 79:323--380,
  2017.
\newblock with discussion.

\bibitem{DrtonBook2008LecAlgStat}
Mathias Drton, Bernd Sturmfels, and Seth Sullivant.
\newblock {\em Lectures on algebraic statistics}, volume~39.
\newblock Springer Science \& Business Media, 2008.

\bibitem{GoodfellowDLBook2016}
Ian Goodfellow, Yoshua Bengio, and Aaron Courville.
\newblock {\em Deep Learning}.
\newblock MIT Press, 2016.

\bibitem{nhayashi8}
Naoki Hayashi.
\newblock Variational approximation error in non-negative matrix factorization.
\newblock {\em Neural Networks}, 126:65--75, 2020.

\bibitem{nhayashi9}
Naoki Hayashi.
\newblock The exact asymptotic form of bayesian generalization error in latent
  dirichlet allocation.
\newblock {\em Neural Networks}, 137:127--137, 2021.

\bibitem{NHayashiPhDThesis}
Naoki Hayashi.
\newblock {\em Statistical Learning Theory of Parameter-Restricted Singular
  Models}.
\newblock PhD thesis, Tokyo Institute of Technology, 2021.

\bibitem{nhayashi5}
Naoki Hayashi and Sumio Watanabe.
\newblock Tighter upper bound of real log canonical threshold of non-negative
  matrix factorization and its application to bayesian inference.
\newblock In {\em IEEE Symposium Series on Computational Intelligence (IEEE
  SSCI)}, pages 718--725, 11 2017.

\bibitem{nhayashi2}
Naoki Hayashi and Sumio Watanabe.
\newblock Upper bound of bayesian generalization error in non-negative matrix
  factorization.
\newblock {\em Neurocomputing}, 266C(29 November):21--28, 2017.

\bibitem{nhayashi7}
Naoki Hayashi and Sumio Watanabe.
\newblock Asymptotic bayesian generalization error in latent dirichlet
  allocation and stochastic matrix factorization.
\newblock {\em SN Computer Science}, 1(2):1--22, 2020.

\bibitem{Hironaka}
Heisuke Hironaka.
\newblock Resolution of singularities of an algbraic variety over a field of
  characteristic zero.
\newblock {\em Annals of Mathematics}, 79:109--326, 1964.

\bibitem{Hu2022XMIR}
Brian Hu, Bhavan Vasu, and Anthony Hoogs.
\newblock X-mir: Explainable medical image retrieval.
\newblock In {\em Proceedings of the IEEE/CVF Winter Conference on Applications
  of Computer Vision}, pages 440--450, 2022.

\bibitem{Imai2019estimating}
Toru Imai.
\newblock Estimating real log canonical thresholds.
\newblock {\em arXiv preprint arXiv:1906.01341}, 2019.

\bibitem{Klimiene2022MultiviewCBM}
Ugne Klimiene, Ri{\v{c}}ards Marcinkevi{\v{c}}s, Patricia Reis~Wolfertstetter,
  Ece {\"O}zkan~Elsen, Alyssia Paschke, David Niederberger, Sven Wellmann,
  Christian Knorr, and Julia~E Vogt.
\newblock Multiview concept bottleneck models applied to diagnosing pediatric
  appendicitis.
\newblock In {\em 2nd Workshop on Interpretable Machine Learning in Healthcare
  (IMLH)}, pages 1--15. ETH Zurich, Institute for Machine Learning, 2022.

\bibitem{Koh2020CBM20a}
Pang~Wei Koh, Thao Nguyen, Yew~Siang Tang, Stephen Mussmann, Emma Pierson, Been
  Kim, and Percy Liang.
\newblock Concept bottleneck models.
\newblock In Hal~Daum^^c3^^a9 III and Aarti Singh, editors, {\em Proceedings of
  the 37th International Conference on Machine Learning}, volume 119 of {\em
  Proceedings of Machine Learning Research}, pages 5338--5348. PMLR, 13--18 Jul
  2020.

\bibitem{Kohjima}
Masahiro Kohjima, Tatsushi Matsubayashi, and Hiroshi Sawada.
\newblock Probabilistic non-negative inconsistent-resolution matrices
  factorization.
\newblock In {\em Proceeding of CIKM '15 Proceedings of the 24th ACM
  International on Conference on Information and Knowledge Management},
  volume~1, pages 1855--1858, 2015.

\bibitem{Kohjima2016NMFreview}
Masahiro Kohjima, Tatsushi Matsubayashi, and Hiroshi Sawada.
\newblock Multiple data analysis and non-negative matrix/tensor factorization
  [i]: multiple data analysis and its advances.
\newblock {\em The journal of the Institute of Electronics, Information and
  Communication Engineers (IEICE)}, 99(6):543--550, 2016.
\newblock in Japanese.

\bibitem{Kumar2009CBM}
Neeraj Kumar, Alexander~C Berg, Peter~N Belhumeur, and Shree~K Nayar.
\newblock Attribute and simile classifiers for face verification.
\newblock In {\em 2009 IEEE 12th international conference on computer vision},
  pages 365--372. IEEE, 2009.

\bibitem{Lampert2009CBM}
Christoph~H Lampert, Hannes Nickisch, and Stefan Harmeling.
\newblock Learning to detect unseen object classes by between-class attribute
  transfer.
\newblock In {\em 2009 IEEE conference on computer vision and pattern
  recognition}, pages 951--958. IEEE, 2009.

\bibitem{Lee}
Daniel~D. Lee and H.~Sebastian Seung.
\newblock Learning the parts of objects with nonnegative matrix factorization.
\newblock {\em Nature}, 401:788--791, 1999.

\bibitem{Lockhart2022CBMLeakage}
Joshua Lockhart, Nicolas Marchesotti, Daniele Magazzeni, and Manuela Veloso.
\newblock Towards learning to explain with concept bottleneck models:
  mitigating information leakage.
\newblock In {\em Workshop on Socially Responsible Machine Learning (SRML),
  co-located with ICLR 2022}, volume~1, pages 1--6, 2022.

\bibitem{Mahinpei2021BBCLM}
Anita Mahinpei, Justin Clark, Isaac Lage, Finale Doshi-Velez, and Weiwei Pan.
\newblock Promises and pitfalls of black-box concept learning models.
\newblock In {\em Proceeding at the International Conference on Machine
  Learning: Workshop on Theoretic Foundation, Criticism, and Application Trend
  of Explainable AI}, volume~1, pages 1--13, 2021.

\bibitem{Margeloiu2021CBMIntended}
Andrei Margeloiu, Matthew Ashman, Umang Bhatt, Yanzhi Chen, Mateja Jamnik, and
  Adrian Weller.
\newblock Do concept bottleneck models learn as intended?
\newblock In {\em Proceeding at the Learning Representations: Workshop on
  Responsible AI}, volume~1, pages 1--8, 05 2021.

\bibitem{Matsuda1-e}
Ken Matsuda and Sumio Watanabe.
\newblock Weighted blowup and its application to a mixture of multinomial
  distributions.
\newblock {\em IEICE Transactions}, J86-A(3):278--287, 2003.
\newblock in Japanese.

\bibitem{StatRethinkMcElreath2nd}
Richard McElreath.
\newblock {\em Statistical Rethinking: A Bayesian Course with Examples in R and
  Stan}.
\newblock CRC Press, 2nd editon edition, 2020.

\bibitem{McGrath2022AcquisChessAlphaZero}
Thomas McGrath, Andrei Kapishnikov, Nenad Toma^^c5^^a1ev, Adam Pearce, Martin
  Wattenberg, Demis Hassabis, Been Kim, Ulrich Paquet, and Vladimir Kramnik.
\newblock Acquisition of chess knowledge in alphazero.
\newblock {\em Proceedings of the National Academy of Sciences},
  119(47):e2206625119, 2022.

\bibitem{Molnar2020IMLBook}
Christoph Molnar.
\newblock {\em Interpretable machine learning}.
\newblock Lulu. com, 2020.

\bibitem{Nagata2008asymptotic}
Kenji Nagata and Sumio Watanabe.
\newblock Asymptotic behavior of exchange ratio in exchange monte carlo method.
\newblock {\em Neural Networks}, 21(7):980--988, 2008.

\bibitem{Naito2014KFRLCT}
Takuto Naito and Keisuke Yamazaki.
\newblock Asymptotic marginal likelihood on linear dynamical systems.
\newblock {\em IEICE TRANSACTIONS on Information and Systems}, 97(4):884--892,
  2014.

\bibitem{Paatero}
Pentti Paatero and Unto Tapper.
\newblock Positive matrix factorization: A non-negative factor model with
  optimal utilization of error estimates of data values.
\newblock {\em Environmetrics}, 5(2):111--126, 1994.
\newblock doi:10.1002/env.3170050203.

\bibitem{Qian2022StatDynamCBMVideoRep}
Rui Qian, Shuangrui Ding, Xian Liu, and Dahua Lin.
\newblock Static and dynamic concepts for self-supervised video representation
  learning.
\newblock In {\em European Conference on Computer Vision}, pages 145--164.
  Springer, 2022.

\bibitem{Aniruddh2021MLClinicalRiskPred}
Aniruddh Raghu, John Guttag, Katherine Young, Eugene Pomerantsev, Adrian~V.
  Dalca, and Collin~M. Stultz.
\newblock Learning to predict with supporting evidence: Applications to
  clinical risk prediction.
\newblock In {\em Proceedings of the Conference on Health, Inference, and
  Learning}, pages 95^^e2^^80^^93--104, New York, NY, USA, 2021. Association
  for Computing Machinery.

\bibitem{Rusakov2005asymptotic}
Dmitry Rusakov and Dan Geiger.
\newblock Asymptotic model selection for naive bayesian networks.
\newblock {\em Journal of Machine Learning Research}, 6(Jan):1--35, 2005.

\bibitem{SatoK2019PMM}
Kenichiro Sato and Sumio Watanabe.
\newblock Bayesian generalization error of poisson mixture and simplex
  vandermonde matrix type singularity.
\newblock {\em arXiv preprint arXiv:1912.13289}, 2019.

\bibitem{Sato1974zeta}
Mikio Sato and Takuro Shintani.
\newblock On zeta functions associated with prehomogeneous vector spaces.
\newblock {\em Annals of Mathematics}, pages 131--170, 1974.

\bibitem{Sawada2022CBMAUC}
Yoshihide Sawada and Keigo Nakamura.
\newblock Concept bottleneck model with additional unsupervised concepts.
\newblock {\em IEEE Access}, 10:41758--41765, 2022.

\bibitem{Schwarz1978BIC}
Gideon Schwarz.
\newblock Estimating the dimension of a model.
\newblock {\em The annals of statistics}, 6(2):461--464, 1978.

\bibitem{Sharif2014cnn}
Ali Sharif~Razavian, Hossein Azizpour, Josephine Sullivan, and Stefan Carlsson.
\newblock Cnn features off-the-shelf: an astounding baseline for recognition.
\newblock In {\em Proceedings of the IEEE conference on computer vision and
  pattern recognition workshops}, pages 806--813, 2014.

\bibitem{Tanaka2020SwishNNRLCT}
Raiki Tanaka and Sumio Watanabe.
\newblock Real log canonical threshold of three layered neural network with
  swish activation function.
\newblock {\em IEICE Technical Report; IEICE Tech. Rep.}, 119(360):9--15, 2020.

\bibitem{Watanabe1}
Sumio Watanabe.
\newblock Algebraic analysis for non-regular learning machines.
\newblock {\em Advances in Neural Information Processing Systems}, 12:356--362,
  2000.
\newblock Denver, USA.

\bibitem{Watanabe2}
Sumio Watanabe.
\newblock Algebraic geometrical methods for hierarchical learning machines.
\newblock {\em Neural Networks}, 13(4):1049--1060, 2001.

\bibitem{SWatanabeBookE}
Sumio Watanabe.
\newblock {\em Algebraix Geometry and Statistical Learning Theory}.
\newblock Cambridge University Press, 2009.

\bibitem{WatanabeAIC}
Sumio Watanabe.
\newblock Asymptotic equivalence of bayes cross validation and widely
  applicable information criterion in singular learning theory.
\newblock {\em Journal of Machine Learning Research}, 11(Dec):3571--3594, 2010.

\bibitem{WatanabeBIC}
Sumio Watanabe.
\newblock A widely applicable bayesian information criterion.
\newblock {\em Journal of Machine Learning Research}, 14(Mar):867--897, 2013.

\bibitem{SWatanabeBookMath}
Sumio Watanabe.
\newblock {\em Mathematical theory of Bayesian statistics}.
\newblock CRC Press, 2018.

\bibitem{Watanabe2023AdjustCV}
Sumio Watanabe.
\newblock Mathematical theory of bayesian statistics for unknown information
  source.
\newblock {\em Philosophical Transactions of the Royal Society A}, pages 1--26,
  2023.
\newblock to apperar.

\bibitem{WatanabeT2022MultiMixRLCT}
Takumi Watanabe and Sumio Watanabe.
\newblock Asymptotic behavior of bayesian generalization error in multinomial
  mixtures.
\newblock {\em arXiv preprint arXiv:2203.06884}, 2022.

\bibitem{Xu2020MultiTaskConcept}
Yiran Xu, Xiaoyin Yang, Lihang Gong, Hsuan-Chu Lin, Tz-Ying Wu, Yunsheng Li,
  and Nuno Vasconcelos.
\newblock Explainable object-induced action decision for autonomous vehicles.
\newblock In {\em Proceedings of the IEEE/CVF Conference on Computer Vision and
  Pattern Recognition}, pages 9523--9532, 2020.

\bibitem{YamazakiPhDThesis}
Keisuke Yamazaki.
\newblock {\em Asymptotic Expansion of Stochastic Complexities in Singular
  Learning Machines}.
\newblock PhD thesis, Tokyo Institute of Technology, 2003.

\bibitem{Yamazaki1}
Keisuke Yamazaki and Sumio Watanabe.
\newblock Singularities in mixture models and upper bounds of stochastic
  complexity.
\newblock {\em Neural Networks}, 16(7):1029--1038, 2003.

\bibitem{Yamazaki3}
Keisuke Yamazaki and Sumio Watanabe.
\newblock Stochastic complexity of bayesian networks.
\newblock In {\em Uncertainty in Artificial Intelligence (UAI'03)}, pages
  592--599, 2003.

\bibitem{Yamazaki2004BinMixRLCT}
Keisuke Yamazaki and Sumio Watanabe.
\newblock Newton diagram and stochastic complexity in mixture of binomial
  distributions.
\newblock In {\em International Conference on Algorithmic Learning Theory},
  pages 350--364. Springer, 2004.

\bibitem{Yamazaki2}
Keisuke Yamazaki and Sumio Watanabe.
\newblock Algebraic geometry and stochastic complexity of hidden markov models.
\newblock {\em Neurocomputing}, 69:62--84, 2005.
\newblock issue 1-3.

\bibitem{Yamazaki4}
Keisuke Yamazaki and Sumio Watanabe.
\newblock Singularities in complete bipartite graph-type boltzmann machines and
  upper bounds of stochastic complexities.
\newblock {\em IEEE Transactions on Neural Networks}, 16:312--324, 2005.
\newblock issue 2.

\bibitem{Yosinski2014transferable}
Jason Yosinski, Jeff Clune, Yoshua Bengio, and Hod Lipson.
\newblock How transferable are features in deep neural networks?
\newblock In Z.~Ghahramani, M.~Welling, C.~Cortes, N.~Lawrence, and K.Q.
  Weinberger, editors, {\em Advances in Neural Information Processing Systems},
  volume~27, pages 1--9. Curran Associates, Inc., 2014.

\bibitem{Zwiernik2011asymptotic}
Piotr Zwiernik.
\newblock An asymptotic behaviour of the marginal likelihood for general markov
  models.
\newblock {\em Journal of Machine Learning Research}, 12(Nov):3283--3310, 2011.

\end{thebibliography}


\begin{thebibliography}{99}% 文献数が10未満の時 {9}
\bibitem{MAoyagi1}
Aoyagi, M. and Watanabe, S. (2005). "Stochastic complexities of reduced rank regression in Bayesian estimation" Neural Networks，No.18，pp.924-933.
%縮小ランク回帰の一般RLCT
\bibitem{nhayashi-sncs-lda}
Hayashi, N. and Watanabe, S. (2020). "Asymptotic Bayesian Generalization Error in Latent Dirichlet Allocation and Stochastic Matrix Factorization", SN Computer Science, Volume 1, 69 (2020), pp.1-22.
%LDAとSMFのRLCTの上界
\end{thebibliography}
\bibliographystyle{plain}
\begin{comment}

\end{comment}

\end{document}